\documentclass{article}

\usepackage[final]{corl_2026}

\usepackage{graphicx}
\usepackage{subcaption}
\usepackage{amsmath}
\usepackage{amssymb}
\usepackage{amsthm}
\usepackage{booktabs}
\usepackage{cleveref}
\usepackage{algpseudocode}
\usepackage{enumitem}
\usepackage{algorithmicx,algorithm}
\usepackage{etoc}
\usepackage{float}
\usepackage{wrapfig}
\newtheorem{mythm}{\bf Theorem}
\newtheorem{myprob}{\bf Problem}
\newtheorem{mylem}{\bf Lemma}
\newtheorem{mycol}{\bf Corollary}
\newtheorem{myprop}{\bf Proposition}

\newtheorem{remark}{Remark}

\title{CoSMO: Conformal Semantic Mapping for Open-World Safe Robot Navigation}

\author{
  Shuqi Wang$^{1}$ \quad
  Yiwei Zheng$^{1}$ \quad
  Pingrui Zhang$^{3}$ \quad
  Yue Gao$^{1,2}$ \quad
  Xiang Yin$^{1,\dagger}$\\[3pt]
  $^{1}$Shanghai Jiao Tong University \quad
  $^{2}$Shanghai Innovation Institute \quad
  $^{3}$Fudan University\\
}

\begin{document}
\etocdepthtag.toc{main}
  \maketitle


\begin{abstract}
Safe autonomous navigation requires not only geometric obstacle avoidance but also semantic awareness. Different object categories demand different safety clearances, and the set of relevant categories may expand as the robot explores novel environments.  
Vision-language models (VLMs) offer rich semantic representations for building maps of unknown scenes, yet their confidence scores are poorly calibrated, making raw predictions unreliable for safety-critical planning. 
Conformal prediction (CP) can provide coverage guarantees, but existing methods assume a fixed label set and stationary data, limiting their use in open-world settings.
We propose Conformal Semantic Mapping for Open-World Safe Robot Navigation (CoSMO), a semantic safety mapping framework that transforms VLM embeddings into online-calibrated coverage sets and safety-aware obstacle dilations under distribution shift. By calibrating directly in the embedding space, CoSMO avoids the score rescaling that arises in confidence calibration as the vocabulary expands.
Experiments on the Matterport3D (MP3D) dataset show that CoSMO reduces semantic collisions by 61.6\% relative to a geometry-only baseline, while maintaining formal coverage guarantees even under delayed expert feedback. Real-world experiments on a quadruped robot further demonstrate generalization by achieving long-run empirical coverage of $90.73\%$ against the $90\%$ target. Project website: \url{https://yangming911.github.io/cosmo-page/}.
\end{abstract}



\section{Introduction}

A domestic robot navigating an unfamiliar home must avoid hazards based on what an object is, not merely where it is. For example, it should keep a large clearance from a moving person, a moderate margin from a slippery bathtub, and only geometric clearance from a sofa, where incidental contact is relatively harmless. Thus, safe navigation is inherently a semantic problem: safety requirements vary across object categories. 
In practice, a robot deployed in novel environments cannot anticipate every hazard category it will encounter, making this an open-world challenge~\cite{openworld} where the agent must discover and incorporate new categories along with their safety requirements during deployment, while maintaining safe operation throughout.

Navigating unknown environments requires building a spatial representation online. Classical approaches rely on geometric maps that capture obstacle locations but not their semantics. Recent methods leverage vision-language models (VLMs) to build semantic maps and search for goals in unfamiliar environments. However, VLM confidence scores are often poorly calibrated~\cite{tu2024empirical}, making raw predictions unreliable for safety-critical decisions.  

Conformal prediction (CP)~\cite{shafer2008tutorial} offers a principled way to calibrate predictive uncertainty. It has recently been widely adopted for uncertainty quantification in neural vision models for image classification~\cite{romano2020classification,silvarodriguez2025conformal}.
The most relevant prior work~\cite{sundarsingh2026safe} constructs conformalized semantic maps in safe planning. However, these methods rely on a \emph{closed-set} and \emph{offline} calibration setting. Deploying CP in open-world navigation introduces new challenges that this design cannot address:
\begin{itemize}
    \item[\textbf{C1.}] 
    The robot may encounter a previously unseen label. However, existing conformal semantic maps typically define nonconformity scores from probabilities (e.g., softmax scores) over a fixed label set. Adding a new label changes the normalization and makes previously computed scores incompatible with the old calibrated threshold.
    \item[\textbf{C2.}] 
     The distribution shifts as the robot moves between visually distinct scenes (e.g., bedroom to kitchen) and as its own safety policy alters the visited regions.  A static threshold cannot track these changes, necessitating online recalibration such as Adaptive Conformal Inference (ACI)~\citep{gibbs2021adaptive}. However, unlike offline calibration on split datasets where ground-truth labels are freely available, reliable labels are difficult to obtain in real time during deployment.
\end{itemize}

\begin{figure}[t]
    \centering
    \includegraphics[width=\columnwidth]{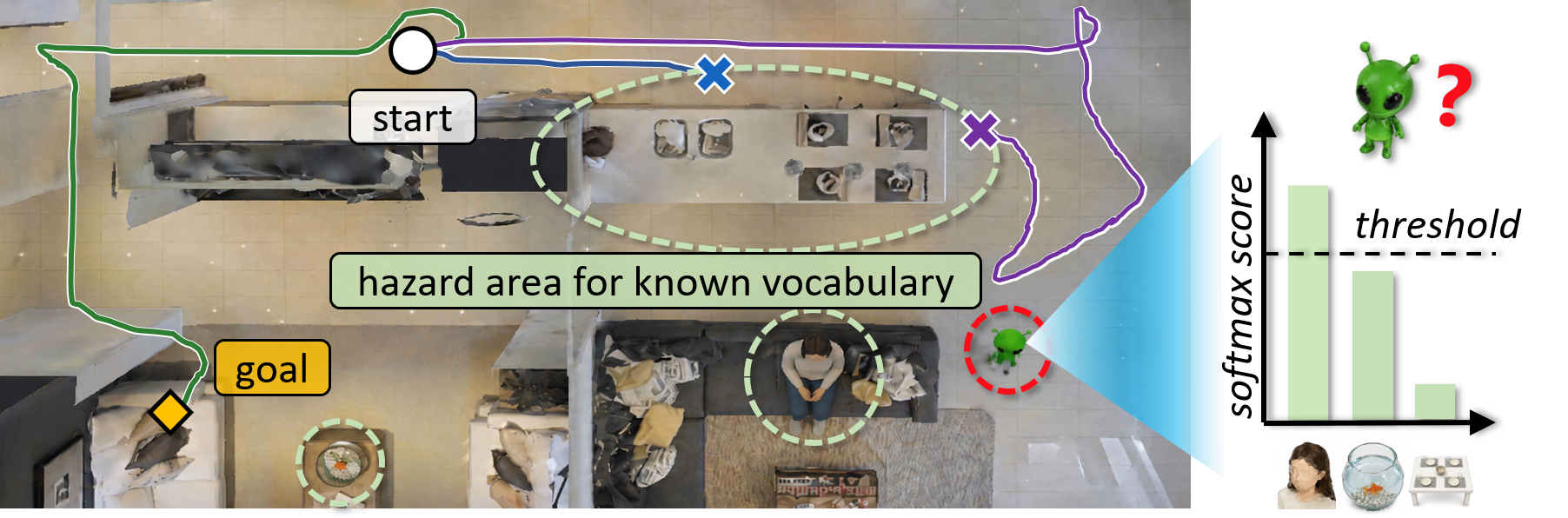}
    \caption{\textcolor[rgb]{0.4,0.0,0.6}{Classical ZSON pipeline}~\citep{busch2025onemap} navigates with no semantic safety awareness and collides. \textcolor[rgb]{0.0,0.2,0.9}{VLM-only Obstacle Avoidance} considers the single most likely label and still enters the semantic hazard zone when misclassified. \textcolor[rgb]{0.0,0.4,0.0}{CoSMO} calibrates coverage sets over the visual embedding, accounting for multiple possible labels, successfully reaching the goal. Closed-set offline conformal prediction breaks when a previously unseen object appears during navigation.}
    \label{fig:senario}
\end{figure}

To address the above challenges, we propose Conformal Semantic Mapping for Open-World Safe Robot Navigation (CoSMO) with three new mechanisms: an embedding-space nonconformity score, an online-maintained safety dictionary, and delayed expert feedback for dictionary expansion and online calibration.
For \textbf{C1}, we define nonconformity scores as distances in the VLM's continuous embedding space, and the agent maintains a dynamically expanding safety dictionary: when the expert identifies a new category, its text embedding is projected into the same space and its safety knowledge is registered in the dictionary. Because distances are structurally independent of vocabulary size, adding new entries never invalidates previously computed scores or the calibrated threshold.
For \textbf{C2}, we update the threshold online via ACI to track score distribution shifts. We employ a cloud-hosted larger model as an expert oracle to provide calibration labels. This introduces a feedback delay. We further extend ACI to the delayed setting, proving that formal coverage guarantees are preserved under bounded delay at a quantifiable tracking cost.
CoSMO can be integrated with existing VLM-based navigation pipelines through three interfaces: (i) a dense embedding vision-language encoder, (ii) a planner that consumes a dilated obstacle map, and (iii) an expert oracle.
%

We demonstrate CoSMO within a zero-shot object navigation (ZSON) 
pipeline~\citep{busch2025onemap}, whose online dense mapping naturally fits CoSMO's integration, and whose diverse object categories provide a realistic testbed for open-world semantic safety. Experiments on 2195 Matterport3D episodes show that CoSMO reduces semantic collisions by 61.6\% relative to a geometry-only ZSON baseline in a fully open-vocabulary setting where every hazard category is discovered online from an initially incomplete dictionary. We further evaluate CoSMO on a Unitree Go2 robot in 30 trials across two indoor scenes with free-form semantic categories. GCLIP Argmax achieves $58.06\%$ top-1 obstacle-label accuracy, whereas CoSMO's prediction sets achieve long-run empirical true-label coverage of $90.73\%$ under delayed Grounding-DINO feedback, meeting the $90\%$ target and demonstrating generalization to physical deployment.

\section{Problem Formulation}
\label{subsec:problem}
The environment is represented as a semantic grid map of $M$ cells indexed by $\mathcal{J} = \{1, \dots, M\}$, where each cell $j$ is centered at $c_j \in \Omega$ and carries a \textit{ground-truth semantic label} $l_j^{(t)} \in \mathcal{V}$ at time $t$. A \emph{semantic safety function} $D_{\text{safe}}: \mathcal{V} \to \mathbb{R}_{\ge 0}$ maps each label to a required minimum clearance. For the agent's current position $x_t$, the robot is \textit{semantically safe} w.r.t.\ cell $j$ if $\| x_t - c_j \|_2 \ge D_{\text{safe}}(l_j^{(t)})$. 
We denote by $\mathcal{N}_r(x_t) \subseteq \mathcal{J}$ the \emph{safety-relevant neighborhood}: the set of cells that the robot could physically interact with within one control step from position $x_t$. Beyond  avoiding occupied cells, the robot must maintain \emph{semantically appropriate} safety margins that depend on the \emph{type} of obstacle.

\begin{myprob}
The robot must navigate to a target object goal while satisfying the following per-step semantic safety condition:
\begin{equation}
    \bigwedge_{j \in \mathcal{N}_r(x_t)} \left( \| x_t - c_j \|_2 \ge D_{\text{safe}}(l_j^{(t)}) \right).
    \label{eq:chance_constraint}
\end{equation}
\end{myprob}

Since the true label $l_j^{(t)}$ is not directly available, the agent must infer it from a VLM-based perception module, whose outputs are unreliable and stochastic. Therefore, raw VLM predictions cannot be directly trusted for safety-guaranteed avoidance.

\begin{figure}[t]
    \centering
    \includegraphics[width=0.9\columnwidth]{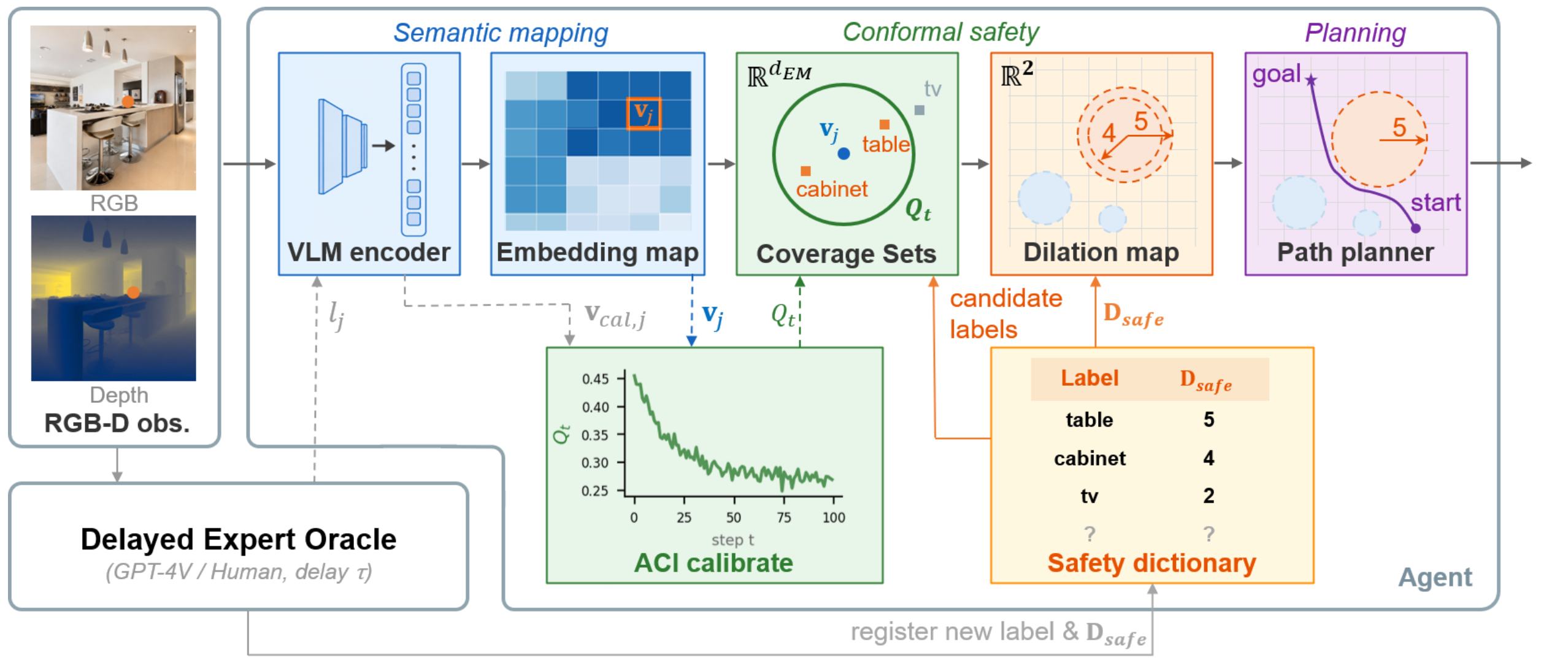}
    \caption{CoSMO module for open-world semantic safety. CoSMO takes an embedding-based semantic map as input, constructs conformal coverage sets, maintains the semantic dilation map, updates the adaptive threshold via open-vocabulary ACI, and expands the safety dictionary using delayed expert feedback. The updated dilation map serves as a safety layer for downstream planning.}
    \label{fig:framework}
\end{figure}

\section{Methodology}
\label{sec:system}
CoSMO takes as input an embedding-based semantic map, the current safety dictionary, and delayed expert feedback. It performs four coupled operations: (i) constructing conformal coverage sets for safety-relevant cells, (ii) updating the semantic obstacle dilation map using worst-case safety radii, (iii) calibrating the adaptive threshold via ACI using delayed nonconformity scores, and (iv) maintaining an open-vocabulary safety dictionary as new labels are revealed. The downstream planner then consumes the resulting dilated obstacle map to produce safety-aware navigation actions. Figure~\ref{fig:framework} summarizes this module-level integration.


\paragraph{Embedding-Based Semantic Map Interface}
\label{subsec:mapping}

We construct a top-down grid map where each cell maintains a high-dimensional semantic feature vector. Let $\phi: \mathbb{R}^{H \times W \times 3} \to \mathbb{R}^{H' \times W' \times d_{\text{em}}}$ be the dense image encoder of a pre-trained VLM (e.g., CLIP), which maps the RGB image $I_t$ to a spatial feature map of $d_{\text{em}}$-dimensional patch embeddings. Let $\psi: \mathcal{V} \to \mathbb{R}^{d_{\text{em}}}$ be the corresponding text encoder, mapping a semantic label to the same embedding space. At each time step, the dense visual features $\phi(I_t)$ are computed and back-projected onto the grid using camera intrinsics and the depth map $Z_t$. Using the depth map $Z_t$, the agent distinguishes free space from occupied cells. For every grid cell $j$ observed by this back-projection at time $t$, the map is updated as
\begin{equation}
    \mathbf{v}_j^{(t)} = \phi(I_t)_j,
    \label{eq:vlm_update}
\end{equation}
where $\phi(I_t)_j \in \mathbb{R}^{d_{\text{em}}}$ is the visual embedding assigned to cell $j$. The map stores continuous feature vectors rather than discrete label predictions, so \textit{the map never needs to be re-rendered when the vocabulary changes}---new labels are simply queried via $\psi$ on the existing features map.

\paragraph{Delayed Expert Oracle}
\label{subsec:expert}

The agent runs a lightweight, edge-deployed VLM for real-time inference, but such models have limited generalization in open-world scenes. To make the resulting VLM feature map usable for safety-critical planning, we   continuously \emph{calibrate the agent's uncertainty and teach it new safety-relevant categories} during deployment. We  leverage a heavyweight cloud-hosted model (e.g., GPT-4V) or a human-in-the-loop verifier as a \emph{delayed expert oracle}. 
Formally, the expert oracle $f_{\text{expert}}: \mathcal{N}_r(x_t) \to (\mathcal{V} \times \mathbb{R}_{\ge 0})^{|\mathcal{N}_r(x_t)|}$ returns, for the robot's neighborhood at time $t$, the ground-truth semantic labels and safety distances $\{(l_j^{(t)},\, D_{\text{safe}}(l_j^{(t)}))\}_{j \in \mathcal{N}_r(x_t)}$ after a bounded delay $\tau \in \mathbb{N}$. That is, at time $t + \tau$, the label--distance pairs \emph{as they were at time $t$} are revealed. 

Note that, the delayed labels are \emph{not} used to correct individual cell embeddings. Instead, the agent compares the expert labels with its stored embeddings $\{\mathbf{v}_j^{(t-\tau)}\}$ to compute retrospective nonconformity scores for calibration. Thus, the expert oracle plays two roles: it expands the safety dictionary with newly discovered labels and safety distances, and it provides delayed calibration data for online threshold adaptation.
Also, the perfect ground-truth oracle assumption is needed only for the formal coverage guarantees. 
In practice, the oracle can be implemented by a strong open-set detector. 
For example, Grounding-DINO~\citep{liu2024groundingdino} yields performance close to the ground-truth oracle, which is illustrated in \Cref{tab:main_results}.

\paragraph{Open-Vocabulary Safety Dictionary}
\label{subsec:open_vocab}

Upon receiving ground-truth labels from the expert oracle, the agent maintains an obstacle safety dictionary $(\mathcal{V}_t, D_{\text{safe}})$ that starts empty ($\mathcal{V}_0 = \emptyset$) and is expanded online. When the expert identifies a label $l_{\text{new}} \notin \mathcal{V}_t$ for the first time, the text embedding $\psi(l_{\text{new}})$ is computed and appended to the obstacle embedding matrix, the safety distance $D_{\text{safe}}(l_{\text{new}})$ is recorded from the expert, and the dictionary is updated via $\mathcal{V}_{t+1} \leftarrow \mathcal{V}_t \cup \{l_{\text{new}}\}$.
This expansion is \textit{strictly additive}: once a semantic concept and its safety margin are registered, they are immutable. This monotonicity ensures that the robot's navigable space can only become more conservative (or stay the same) when new hazards are discovered, preventing safety regression.


\paragraph{Open-vocabulary Adaptive Conformal Inference}
\label{subsec:aci}
Standard VLM predictions are noisy and poorly calibrated, making raw similarity scores unreliable for safety-critical planning. CoSMO enforces the per-step safety condition~\eqref{eq:chance_constraint} by turning an embedding-based semantic map into conformal coverage sets and then into a conservative semantic dilation map. The key idea is to define nonconformity scores as distances in the VLM's embedding space rather than softmax-normalized probabilities over a fixed label set, so that calibration becomes structurally independent of vocabulary size. We first define the scoring mechanism, then present the online calibration procedure that produces the adaptive threshold $Q_t$, and finally describe how $Q_t$ induces coverage sets and safety-critical planning.

The threshold $Q_t$ is calibrated online using the ground-truth labels provided by the expert oracle via Adaptive Conformal Inference (ACI)~\citep{gibbs2021adaptive}. We adopt a pairwise embedding-space dissimilarity $s(\mathbf{v}, \mathbf{v}_{\text{cal}})$ (e.g., cosine distance) as the per-cell nonconformity score, where $\mathbf{v}_{\text{cal},j}^{(t)} = \psi(l_j^{(t)})$ is the text embedding of the ground-truth label of cell $j$. Because $s$ depends only on the pair $(\mathbf{v}, \mathbf{v}_{\text{cal}})$, adding a new label to $\mathcal{V}_t$ does not alter any previously computed score, so the score is \emph{vocabulary-agnostic}. We aggregate $s$ over the robot's neighborhood into the worst-case scalar $S_t = \max_{j \in \mathcal{N}_r(x_t)} s(\mathbf{v}_j^{(t)},\, \mathbf{v}_{\text{cal},j}^{(t)})$, which inherits the same vocabulary-agnosticity.

The robot maintains a risk parameter $\alpha_t$ that governs how conservative the coverage sets are. At time $t$, the expert oracle delivers the ground-truth labels $\{l_j^{(t-\tau)}\}_{j \in \mathcal{N}_r(x_{t-\tau})}$ for all cells in the neighborhood observed at time $t - \tau$. The agent forms the retrospective neighborhood score $S_{t-\tau} = \max_{j \in \mathcal{N}_r(x_{t-\tau})} s(\mathbf{v}_j^{(t-\tau)},\, \mathbf{v}_{\text{cal},j}^{(t-\tau)})$ and runs the ACI update
\begin{equation}
    \text{err}_t = \mathbb{I}(S_{t-\tau} > Q_t),\quad
    \alpha_{t+1} = \alpha_t + \gamma(\alpha - \text{err}_t),\quad
    Q_{t+1} = \hat{Q}(1 - \alpha_{t+1}),
    \label{eq:aci_update}
\end{equation}
where $\gamma > 0$ is the step size, $\alpha$ is the target miscoverage rate, and $\hat{Q}(q)$ is the $q$-th empirical quantile of historical neighborhood nonconformity scores. The dynamics form a negative feedback loop: when all cells in the neighborhood are covered ($\text{err}_t = 0$), $\alpha_t$ increases, $Q_t$ decreases, and coverage sets shrink; when any cell's coverage fails ($\text{err}_t = 1$), $\alpha_t$ decreases, $Q_t$ increases, and coverage sets expand. This drives long-run neighborhood coverage toward $1 - \alpha$.

\paragraph{Coverage Sets and Safety-Critical Planning}
\label{subsec:conf_set}

Given the adaptive threshold $Q_t$ from ACI, the coverage set for cell $j$ is
\begin{equation}
    \mathcal{C}_j^{(t)} = \left\{ l \in \mathcal{V}_t \;\middle|\; s(\mathbf{v}_j^{(t)}, \psi(l)) \le Q_t \right\}.
    \label{eq:conf_set}
\end{equation}
For $\mathcal{C}_j^{(t)} \neq \emptyset$, CoSMO assigns the worst-case semantic safety distance $d_{\text{wc}}^{(j,t)} = \max_{l \in \mathcal{C}_j^{(t)}} D_{\text{safe}}(l)$. When $\mathcal{C}_j^{(t)} = \emptyset$, CoSMO sets $d_{\text{wc}}^{(j,t)} = 0$, applies no semantic dilation to that cell, and relies on the base geometric obstacle map. This case occurs before the true label has entered $\mathcal{V}_t$, or when no discovered label lies within the calibrated threshold.

By construction of $S_t$, the event $\{S_t \le Q_t\}$ is equivalent to $\{l_j^{(t)} \in \mathcal{C}_j^{(t)} \text{ for all } j \in \mathcal{N}_r(x_t)\}$. Whenever this event holds, $d_{\text{wc}}^{(j,t)} = \max_{l \in \mathcal{C}_j^{(t)}} D_{\text{safe}}(l) \ge D_{\text{safe}}(l_j^{(t)})$ for every cell, so planning against $d_{\text{wc}}^{(j,t)}$ enforces the per-step safety condition~\eqref{eq:chance_constraint}. The role of the ACI update is to keep $Q_t$ large enough that this current-time coverage event holds frequently, even though calibration only sees the delayed score $S_{t-\tau}$ rather than $S_t$; quantifying how delay and non-stationarity affect this guarantee is the subject of \Cref{sec:analysis}.

The output of CoSMO is an obstacle map in which each obstacle cell is dilated by $d_{\text{wc}}^{(j, t)}$. One can feed this map into a standard graph-based planner that returns safe trajectories. 


\section{Theoretical Analysis}
\label{sec:analysis}

We provide theoretical guarantees for CoSMO. These guarantees apply in the post-discovery regime, where the true label of every safety-relevant cell lies in the current vocabulary $\mathcal{V}_t$. 
Proofs are given in Appendix~\ref{app:proofs}. 

\subsection{Coverage Guarantees}

Under delayed feedback, the coverage error at time $t$ is evaluated using the \emph{current} threshold $Q_t$ against the delayed neighborhood score $S_{t-\tau}$. Because the ACI update  has the same algebraic form regardless of \emph{which} neighborhood is evaluated, 
one can derive the following result similar to the standard ACI guarantee~\citep{gibbs2021adaptive}.

\begin{myprop}[\textbf{Asymptotic Calibration under Delay}]
\label{prop:coverage_delay}
For all $T \ge 1$, it holds that
$$\left| \frac{1}{T} \sum_{t=1}^T \emph{err}_t - \alpha \right| \le \frac{\max\{\alpha_1,\, 1-\alpha_1\} + \gamma}{T\gamma}.$$
\end{myprop}

Note that, the above bound requires \emph{no distributional assumption} and contains \emph{no dependence on $\tau$}.
To bridge calibration coverage (on delayed neighborhoods) to planning-time coverage (on current neighborhoods), we introduce two regularity conditions. Let $F_t(q) \triangleq \mathbb{P}(S_t \le q)$ denote the CDF of the neighborhood nonconformity score at time $t$. We further make the following two assumptions:\vspace{-6pt}
\begin{itemize}
    \item[\textbf{A1}] \textbf{(Bounded Score Distribution Shift)} 
    There exists $\Delta \ge 0$ such that $\sup_{q \ge 0} |F_{t+1}(q) - F_t(q)| \le \Delta$ for all $t \ge 1$.
    
    \item[\textbf{A2}] \textbf{(Bounded Conditional Bias)} 
    There exists $\epsilon \ge 0$ such that for all $t \ge 1$:
$|\mathbb{P}(S_{t-\tau} > Q_t \mid \mathcal{F}_{t-1}) - (1 - F_{t-\tau}(Q_t))| \le \epsilon$ and
$|\mathbb{P}(S_t > Q_t \mid \mathcal{F}_{t-1}) - (1 - F_t(Q_t))| \le \epsilon$,
where $\mathcal{F}_{t-1} = \sigma(\mathrm{err}_1, \dots, \mathrm{err}_{t-1})$.
\end{itemize}

\begin{myprop}[\textbf{Planning-Time Coverage under Delay}]
\label{prop:tracking}
Under Assumptions \textbf{A1} and \textbf{A2} and the delayed expert oracle, the expected planning-time miscoverage satisfies
\begin{equation}
    \limsup_{T \to \infty}\;\left| \frac{1}{T}\sum_{t=1}^{T} \bigl(1 - F_t(Q_t)\bigr) - \alpha \right| \le \tau \Delta + \epsilon.
    \label{eq:tracking_bound}
\end{equation}
\end{myprop}

The bound decomposes into a temporal gap $\tau\Delta$ (vanishes when $\tau=0$ or the score distribution is stationary) and a dependence bias $\epsilon$ (vanishes when successive neighborhoods do not overlap). Notably, $\Delta$ absorbs \emph{all} sources of score non-stationarity---including time-varying labels $l_j^{(t)}$---so the guarantees hold in dynamic environments without requiring static-world assumptions.

\subsection{Probabilistic Joint Safety Guarantee}

We now link the neighborhood coverage to the physical safety constraints~\eqref{eq:chance_constraint}. For this guarantee, we assume that the RGB-D observation $(I_t, Z_t)$ covers $\mathcal{N}_r(x_t)$ at each step, so every safety-relevant cell has an up-to-date embedding when the safety map is constructed. The neighborhood radius is chosen so that $r \ge \bar{D} + \delta$, where $\delta$ upper-bounds the per-step displacement $\|x_{t+1} - x_t\|$. This condition guarantees that every cell whose safety zone the robot could enter at $x_{t+1}$ is already contained in $\mathcal{N}_r(x_t)$, so the safety map computed at time $t$ remains valid for planning the next step.

\begin{mythm}[\textbf{Per-Step Semantic Safety}]
\label{thm:safety}
With reliable obstacle detection and $r \ge \bar{D} + \delta$, suppose that at time $t$:
(i) $\mathbb{P}(S_t \le Q_t) \ge 1 - \alpha_t'$ for some $\alpha_t' \in (0,1)$; and
(ii) the robot state $x_t$ satisfies $\| x_t - c_j \| \ge d_{\emph{wc}}^{(j,t)}$ for all $j \in \mathcal{J}_{\emph{obs}}^{(t)} \cap \mathcal{N}_r(x_t)$.
Then $\mathbb{P}(\exists\, j \in \mathcal{N}_r(x_t) : \| x_t - c_j \| < D_{\emph{safe}}(l_j^{(t)})) \le \alpha_t'$.
\end{mythm}

Unlike per-cell CP where a union bound over $M$ cells requires replacing $\alpha$ by $\alpha/M$, our neighborhood-aggregated score $S_t$ provides \emph{joint} coverage directly. Combining Theorem~\ref{thm:safety} with Proposition~\ref{prop:tracking} yields:

\begin{mycol}[\textbf{Long-Run Average Safety}]
\label{cor:mission}
Under reliable obstacle detection, Assumptions \textbf{A1} and \textbf{A2}, and the delayed expert oracle, if the planner ensures $\| x_t - c_j \| \ge d_{\emph{wc}}^{(j,t)}$ for all $j \in \mathcal{J}_{\emph{obs}}^{(t)} \cap \mathcal{N}_r(x_t)$ at every step, then:
\begin{equation}
    \limsup_{T \to \infty}\; \frac{1}{T}\sum_{t=1}^{T} \mathbb{P}\!\left(\exists\, j \in \mathcal{N}_r(x_t) : \| x_t - c_j \| < D_{\emph{safe}}(l_j^{(t)}) \;\middle|\; Q_t \right) \le \alpha + \tau\Delta + 2\epsilon.
    \label{eq:mission_safety}
\end{equation}
The bound degrades gracefully: it reduces to $\alpha$ when $\tau = 0$ or $\Delta = 0$, and when $\epsilon = 0$.
\end{mycol}

\begin{remark}[Practical magnitude of $\epsilon$]
\label{rem:epsilon}
$\epsilon$ could be non-negligible when consecutive neighborhoods are temporally correlated. Empirically, however, the delay ablation in \Cref{fig:delay_ablation} 
indicates that the $\epsilon$ term has negligible practical effect in our experimental setting, which covers diverse indoor scenes representative of general domestic deployment.
\end{remark}


\section{Experiments}
\label{sec:experiments}


\subsection{Setup}

\begin{wrapfigure}{r}{0.35\columnwidth}
    \vspace{-5pt}
    \centering
    \includegraphics[width=0.25\columnwidth]{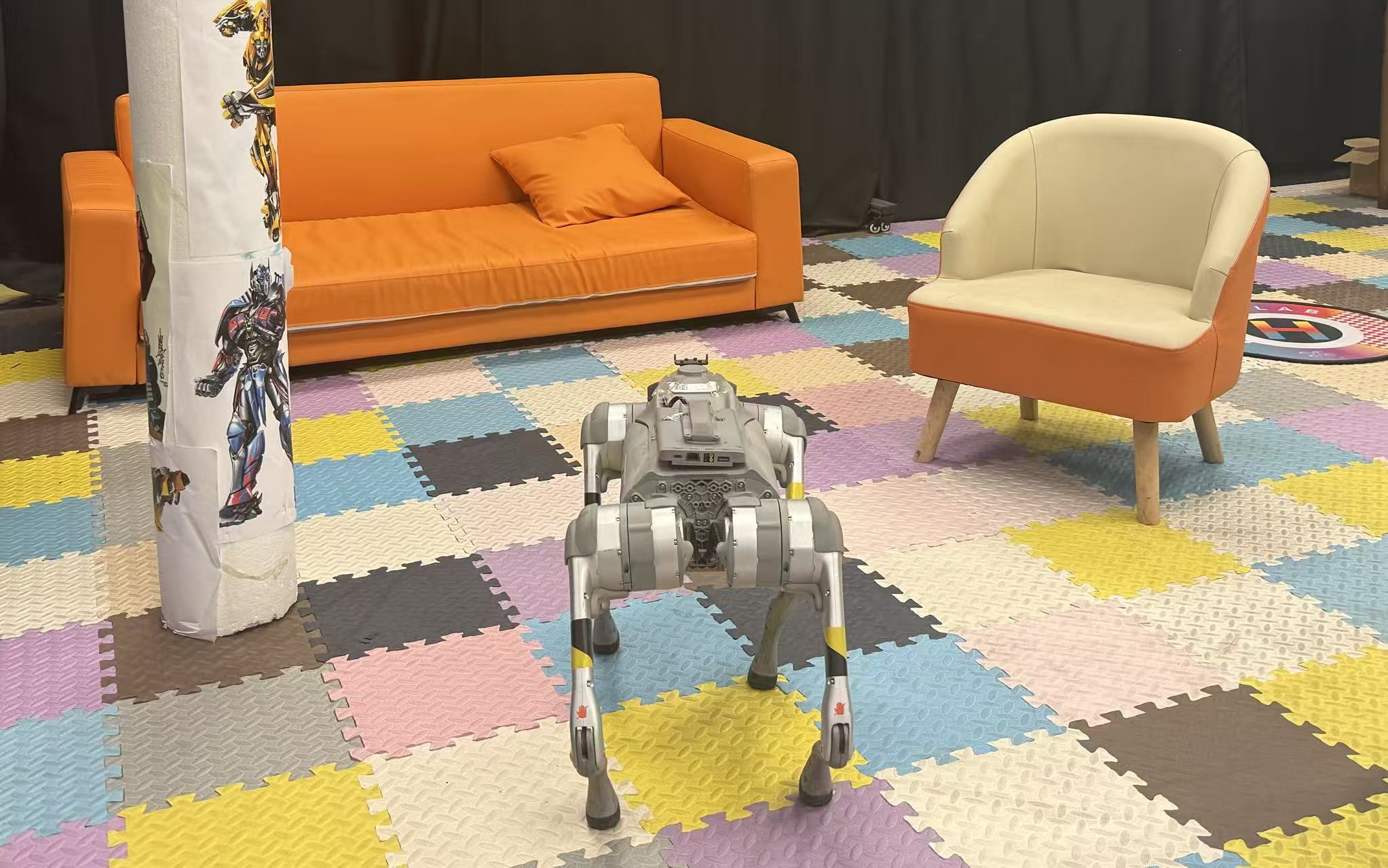}
    \caption{Real-world deployment.}
    \label{fig:real_robot}
    \vspace{-7pt}
\end{wrapfigure}
We use full 2195 object navigation episodes from Matterport3D (MP3D) dataset~\citep{chang2017matterport3d}. Each episode specifies a sequence of target objects; the robot must navigate to each in turn. We report three metrics: \textit{Success Rate} (SR), the fraction of episodes where all targets are reached; \textit{Semantic Collision Rate} (SCR), the fraction of episodes terminated due to violating a semantic safety constraint; and \textit{Success weighted by Path Length} (SPL)~\citep{anderson2018evaluation}, defined as $\frac{1}{N}\sum_{i=1}^N S_i \frac{d_i}{\max(p_i, d_i)}$ where $S_i$ is a binary success indicator, $d_i$ is the geodesic distance, and $p_i$ is the actual path length.
Unlike conventional ObjectNav evaluation, where physical contacts with obstacles are often handled as blocked motion, our protocol treats semantic safety violations as episode-level failures. An episode is terminated whenever the robot enters a label-dependent safety region. Accordingly, SR and SPL reflect performance in a more stringent safety-aware navigation setting, where goal-reaching and semantic clearance are evaluated jointly.

We compare six conditions on the same 2195 episodes. The obstacle safety dictionary and semantic distance assignments are detailed in \Cref{app:obstacle_dict}. In all conditions, the robot has \textbf{no prior knowledge} of the environment. Conditions B--D use GCLIP dense features for obstacle detection; B--D start with $\mathcal{V}_0 = \emptyset$. 
\textbf{A.} \textbf{Baseline (OneMap):} Standard OneMap~\citep{busch2025onemap} navigation with YOLO-based obstacle detection. No semantic safety awareness.
\textbf{B.} \textbf{GCLIP Argmax:} VLM dense features for obstacle detection with no conformal prediction.
\textbf{C.} \textbf{Uniform Dilation:} All detected obstacle cells are dilated by the maximum safety distance $\bar{D}$, regardless of their semantic label.
\textbf{D.} \textbf{CoSMO (open-vocabulary):} Full CoSMO pipeline with ACI-calibrated threshold ($\alpha = 0.1$, $\gamma = 0.05$). The dictionary starts empty and is expanded online.
\textbf{E.} \textbf{CoSMO (known-vocabulary):} Same as D, but the full obstacle dictionary is pre-loaded at the start of each episode ($\mathcal{V}_0 = \mathcal{V}$). This isolates the effect of cold-start dictionary expansion from the conformal calibration.
\textbf{F.} \textbf{CoSMO (calibrated with Grounding-DINO):} Same as D, but the expert oracle is replaced by Grounding DINO as the calibration source. This evaluates CoSMO under a realistic, noisy expert.
Real-world evaluation is presented in Section~\ref{subsec:real_world_experiments}.
\subsection{Main Results}

\begin{table}[t]
\centering
\begin{minipage}[b]{0.5\columnwidth}
    \centering
    \setlength{\tabcolsep}{3pt}
    \renewcommand{\arraystretch}{1.15}
    \begin{tabular}{@{}lccc@{}}
    \toprule
    \textbf{Method} & \textbf{SR} $\uparrow$ & \textbf{SCR} $\downarrow$ & \textbf{SPL} $\uparrow$ \\
    \midrule
    A. Baseline (OneMap)     & 11.3\% & 25.5\%          & 0.071 \\
    B. GCLIP Argmax          & 18.4\% & 13.1\%          & 0.103 \\
    C. Uniform Dilation      & 8.7\%  & 8.9\%           & 0.053 \\ 
    D. CoSMO (open-vocab.)    & 13.8\% & 9.8\%           & 0.081 \\
    \midrule
    E. CoSMO (known-vocab.)   & 10.4\% & 9.4\%  & 0.059 \\
    F. CoSMO (calib.\ GDINO)  & 16.5\% & 10.5\%          & 0.097 \\
    \bottomrule
    \end{tabular}
    \vspace{0.8em}
    \captionof{table}{Navigation results on the full MP3D val split (2195 episodes). SR = Success Rate, SCR = Semantic Collision Rate, SPL = Success weighted by Path Length. Best safety in \textbf{bold}.}
    \label{tab:main_results}
\end{minipage}\hfill
\begin{minipage}[b]{0.46\columnwidth}
    \centering
    \includegraphics[width=0.95\linewidth]{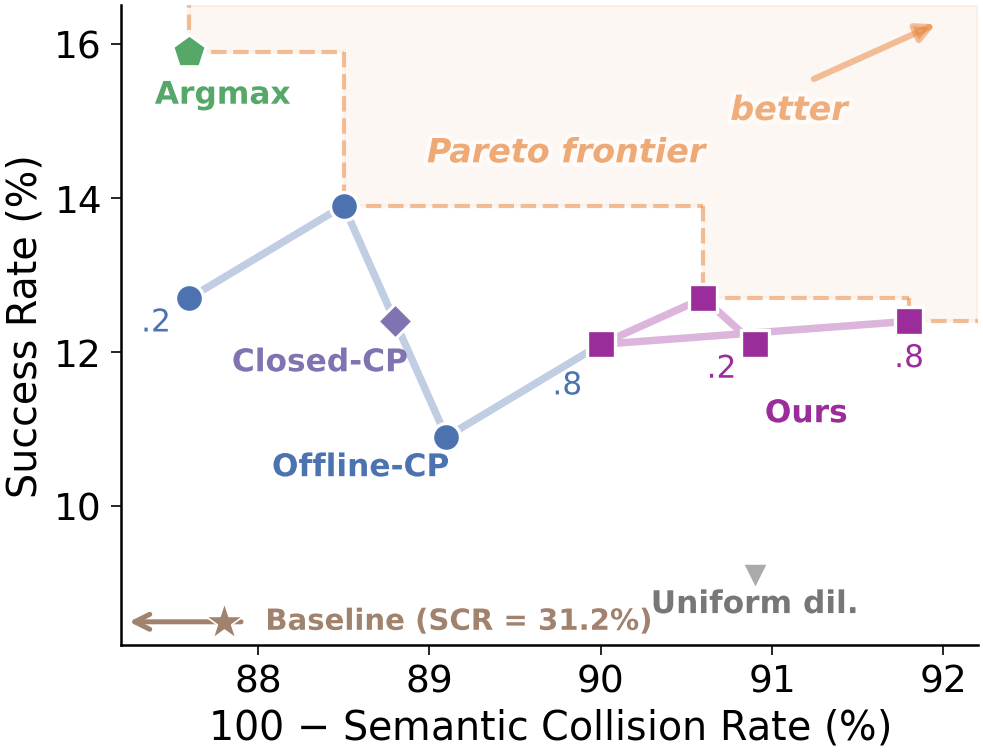}
    \vspace{0.3em}
    \captionof{figure}{Pareto frontier of safety ($100 - \text{SCR}$) vs.\ Success Rate on 330 ablation episodes.}
    \label{fig:pareto}
\end{minipage}
\end{table}

\Cref{tab:main_results} presents the main comparison. Conditions B--D and F operate under \textit{full vocabulary hold-out}: the robot starts each episode with $\mathcal{V}_0 = \emptyset$ and discovers all 5 obstacle categories online. Condition E pre-loads the full dictionary to isolate the cold-start effect. Detailed analysis is provided in Appendix \Cref{app:experiments}.

The Baseline yields $25.5\%$ SCR. GCLIP Argmax reduces this to $13.1\%$ by leveraging VLM embeddings for online obstacle detection. \textbf{Conformal calibration further reduces collisions.}
Comparing B and D (same backbone, without/with CoSMO), CoSMO reduces SCR from $13.1\%$ to $\mathbf{9.8\%}$. The end-to-end reduction relative to the geometry-only baseline is $61.6\%$, at a modest SR cost reflecting the safety--efficiency tradeoff. The safety--success tradeoff is controlled by the coverage target $1 - \alpha$. This tradeoff also varies across scenes: in a more open layout, CoSMO reduces SCR by $40\%$ relative to GCLIP Argmax with only a $0.5$ percentage-point decrease in SR (Appendix~\ref{app:scene_tradeoff}). \Cref{fig:pareto} shows the Pareto frontier on 330 ablation episodes. CoSMO uniformly dominates Offline-CP across all coverage levels. The ``Uniform dilation'' baseline achieves the most conservative SCR but at a steep SR cost, isolating the value of label-specific conformal calibration over worst-case uniform dilation. \textbf{Known- vs.\ open-vocabulary.} Pre-loading the dictionary (D~$\to$~E) improves SCR only marginally, showing that CoSMO remains effective under the stricter cold-start setting. \textbf{Closed-set vs.\ open-vocabulary.} At equal coverage targets, no closed-set CP variant matches CoSMO on either SR or SCR (\Cref{fig:pareto}).

\subsection{Sensitivity to Expert Oracle}
\label{subsec:expert_ablation}
\textbf{Expert accuracy.}\quad
Replacing the ground-truth expert with Grounding-DINO~\citep{liu2024groundingdino}, an open-set object detector, increases SCR only marginally from $9.8\%$ to $10.5\%$ on the main 2195-episode benchmark (Condition F, \Cref{tab:main_results}); the method degrades gracefully rather than collapsing (detailed analysis in \Cref{app:expert_quality}).
\textbf{Expert delay.}\quad
\Cref{fig:delay_ablation} shows the effect of $\tau \in \{0, 5, 10, 20, 100\}$ on CoSMO. All $\tau$ values converge to the same $Q_{90\%} \approx 0.257$, and episode-level SCR stays within $10.0$--$10.9\%$ (\Cref{tab:delay_ablation}), confirming delay-independent asymptotic calibration (Proposition~\ref{prop:coverage_delay}). Instantaneous coverage oscillations grow with $\tau$, consistent with Proposition~\ref{prop:tracking}.

\begin{figure}[t]
    \centering
    \includegraphics[width=\columnwidth]{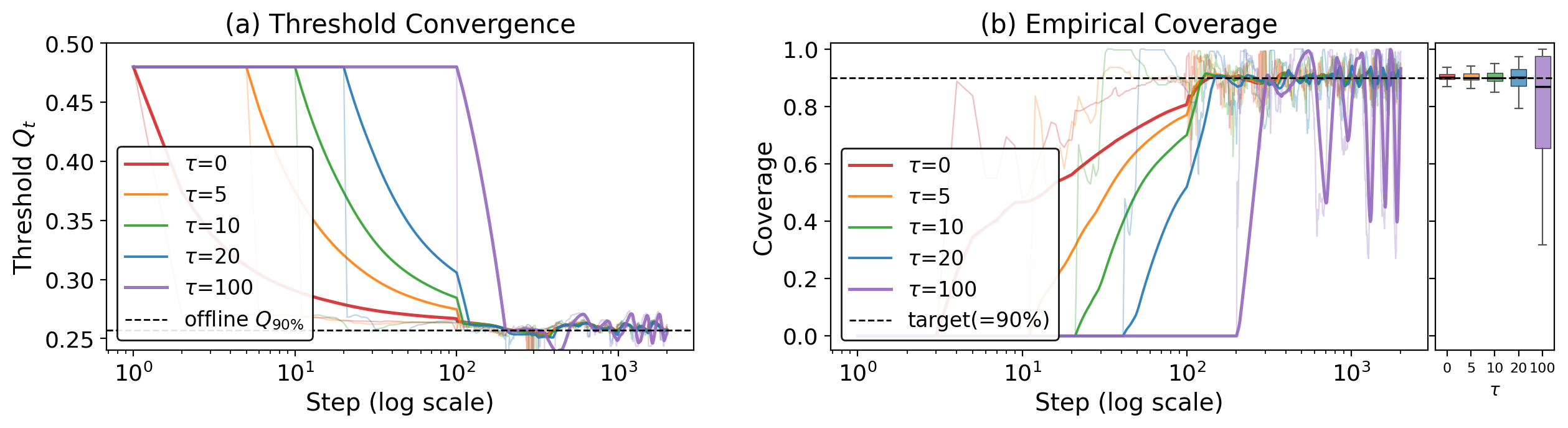}
    \caption{\textbf{Effect of expert delay $\tau$ on CoSMO.} \textbf{(a)}~Threshold convergence: all delay values converge to the same offline quantile $Q_{90\%}$, consistent with Proposition~\ref{prop:coverage_delay}. \textbf{(b)}~Empirical coverage: all curves converge to the $90\%$ target, but larger $\tau$ produces larger oscillations, consistent with Proposition~\ref{prop:tracking}.}
    \label{fig:delay_ablation}
\end{figure}

\subsection{Real-world Experiments}
\label{subsec:real_world_experiments}

We deploy CoSMO on a Unitree Go2 robot (\Cref{fig:real_robot}) and conduct 30 navigation trials across two indoor scenes. The semantic obstacle dictionary in Scenario 1 includes \texttt{comic book}, \texttt{chair}, \texttt{teddy bear}, \texttt{toy}, \texttt{cat}, and \texttt{dog}. Scenario 2 includes \texttt{comic book}, \texttt{chair}, \texttt{Batman}, \texttt{Superman}, \texttt{Spider Man}, \texttt{Iron Man}, and \texttt{Captain America}. These scenarios evaluate both household categories and free-form concepts. The robot searches for a \texttt{chair} while avoiding semantic obstacles, using GCLIP for onboard perception and Grounding-DINO as the delayed expert oracle. The oracle latency is approximately $350$\,ms, and the feedback delay is $128.1$--$447.6$\,ms.

\Cref{fig:real_world_navigation} illustrates some trials in Scenario 2, which is an initially unknown indoor environment. Across the 30 trials, CoSMO achieves an SR of $46.67\%$ and an SCR of $16.67\%$. GCLIP Argmax achieves $58.06\%$ top-1 obstacle-label accuracy, whereas CoSMO's calibrated prediction sets achieve $90.73\%$ empirical true-label coverage, close to the $90\%$ target. These results demonstrate that CoSMO supports semantic safety-aware navigation with delayed detector feedback and diverse unseen semantic categories in physical deployment.

\begin{figure}[t]
    \centering
    \includegraphics[width=\linewidth]{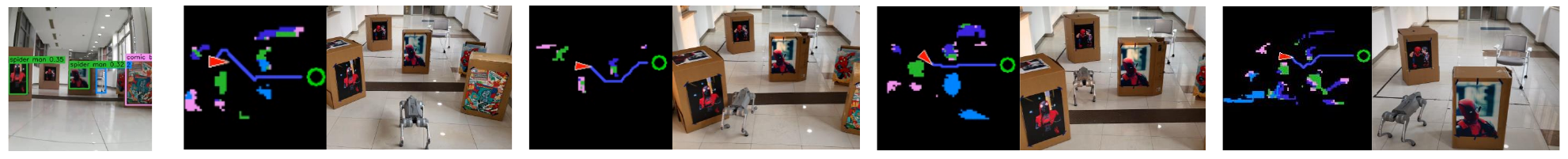}
    \caption{Real-world navigation with a Unitree Go2 robot in an initially unknown $3.5\,\mathrm{m}\times10\,\mathrm{m}$ indoor environment. The robot searches for a chair while avoiding semantic obstacles, with GCLIP for perception and Grounding-DINO for delayed expert feedback. Different colors denote different argmax labels predicted by the VLM-based perception module. The varying argmax labels across observations show the ambiguity of VLM perception. CoSMO instead constructs calibrated coverage sets through online conformal inference, achieving label coverage with a guarantee.}
    \label{fig:real_world_navigation}
\end{figure}


\section{Limitations}
\label{sec:discussion}
Although CoSMO provides a practical way to incorporate uncertainty and semantic safety into robot navigation, its theoretical guarantees and practical deployment are subject to several assumptions and limitations.

The safety guarantees rely on the assumption that the expert oracle provides reliable labels. In practice, the coverage guarantee is only as reliable as the supervision source used for calibration and dictionary expansion. The experiment results of replacing the ground-truth oracle with Grounding-DINO suggest that CoSMO degrades gracefully under noisy supervision. A useful direction for future work is to formally characterize how expert label noise propagates into the coverage guarantee. Moreover, the current implementation queries the expert at every step, which simplifies online calibration and supports the theoretical analysis, but can be costly in real deployments. A promising direction is to design a \emph{self-triggered} querying mechanism in which the agent requests expert feedback only when its own uncertainty is high or when the safety dictionary is insufficient.

A further limitation stems from the evaluation environment itself. Existing ObjectNav benchmarks such as Matterport3D primarily feature compact residential layouts. When semantic safety dilations are introduced, the inflated obstacle regions frequently occlude narrow corridors and doorways, accentuating the SR--SCR tradeoff. The scene-level results in Appendix~\ref{app:scene_tradeoff} illustrate how available free space affects this tradeoff. In the more open scene, CoSMO reduces SCR by $40\%$ relative to  GCLIP Argmax while maintaining comparable SR. In contrast, the more constrained scene exhibits a larger SR decrease. These results demonstrate that CoSMO can substantially improve semantic safety with little loss of task  success when alternative routes around dilated obstacles are available. We expect this tradeoff to be milder in more spacious deployment scenarios such as warehouses, retail stores, or hospital corridors, where semantic clearance zones are less likely  to block the only viable path.

\section{Conclusion}

 We presented CoSMO, a conformal semantic mapping framework for open-world safe robot navigation. By continuously expanding its safety knowledge and adapting to newly encountered object categories, CoSMO remains robust in challenging open-world settings, including unseen categories and distribution shifts. These capabilities establish CoSMO as an open-world calibration framework, supported by formal safety guarantees.
 Experiments on 2195 Matterport3D episodes show that CoSMO reduces semantic collisions by $61.6\%$ relative to the geometry-only baseline, even when all hazard categories are discovered online. Real-world experiments on a Unitree Go2 robot further demonstrate the practicality of CoSMO, maintaining calibrated semantic coverage under practical delayed expert feedback and previously unseen free-form categories.
 Future work should focus on self-triggered expert querying, incorporating expert uncertainty into complete safety guarantees, and establishing dedicated open-world semantic safety benchmarks with spacious environments and native safety specifications.

\clearpage
\acknowledgments{
The authors would like to thank Ruijia Liu, Bowen Ye, and Shaowen Miao for their valuable advice on the writing of this paper.
}

\bibliography{ref}


\clearpage
\appendix
\etocdepthtag.toc{appendix}

\begin{center}
{\Large\bfseries Supplementary Material for the Paper}\\[4pt]
{\large CoSMO: Conformal Semantic Mapping for Open-World Safe Robot Navigation}
\end{center}
\vspace{12pt}

\etocsettagdepth{main}{none}
\etocsettagdepth{appendix}{subsection}
\etocsettocstyle{}{}
\etocsetstyle{section}
  {}
  {\etociffirst{}{\smallskip}}
  {\noindent\makebox[1.8em][l]{\etocnumber.}\etocname\dotfill\etocpage\par}
  {}
\etocsetstyle{subsection}
  {}
  {}
  {\noindent\hspace{1.8em}\makebox[2.2em][l]{\etocnumber.}\etocname\dotfill\etocpage\par}
  {}
\tableofcontents
\vspace{12pt}

\section{Related Works}
\textbf{Vision-Based Navigation.}
Vision-based navigation has been extensively studied in embodied AI, with standardized task definitions and evaluation protocols for agents navigating in unseen 3D environments~\cite{anderson2018evaluation,zhang2026cross,savva2019habitat,zhang2025moma}. In object-goal navigation, the agent must locate an instance of a target object category from egocentric visual observations in an initially unexplored environment. Early semantic navigation methods improve long-horizon exploration by building structured spatial memories: SemExp~\cite{chaplot2020object} constructs an episodic semantic map and uses the goal category to guide exploration, while PONI~\cite{ramakrishnan2022poni} learns object-conditioned potential functions over semantic maps to predict promising search regions. More recently, vision-language models (VLMs) have enabled open-vocabulary navigation without task-specific training. VLMaps~\cite{huang2023vlmaps} anchors dense CLIP~\cite{radford2021learning} features in a spatial map for natural-language goal indexing; CoW~\cite{gadre2023cows} and ZSON~\cite{majumdar2022zson} leverage CLIP similarities for zero-shot object-goal navigation during active exploration; ESC~\cite{zhou2023esc} balances semantic goal pursuit with frontier exploration; and VLFM~\cite{yokoyama2024vlfm} uses vision-language frontier scoring for efficient goal-directed search. More recently, OneMap~\cite{busch2025onemap} introduced a zero-shot multi-object navigation benchmark, where an online open-vocabulary map is reused across consecutive object searches.

Despite this progress, existing vision-based navigation methods primarily focus on goal-reaching performance and exploration efficiency, rather than semantic safety during navigation. For example, in standard ObjectNav evaluation, success is purely based on whether the robot reaches the target object, ignoring collisions with other objects along the way. In real-world deployment, a robot navigating to a ``bed'' should not crash into a ``table'' en route. Our work addresses this gap by introducing semantic safety constraints into the open-vocabulary navigation framework.

\textbf{Conformal Prediction in Uncertainty Quantification.}
Conformal prediction (CP) is a distribution-free framework for constructing prediction sets with finite-sample coverage guarantees \cite{shafer2008tutorial}. Its model-agnostic nature---applicable to any black-box predictor---has made it attractive for safety-critical control, where formal uncertainty quantification is essential.
Existing conformal classification methods mainly design nonconformity scores over a fixed label probability vector. A simple and widely used choice is the softmax/LAC score, which forms prediction sets by thresholding the estimated probability assigned to each candidate class~\cite{sadinle2019least}. This idea is closely related to least-ambiguous set-valued classification, where optimal prediction sets are characterized by level sets of conditional class probabilities. To obtain more adaptive prediction sets, Romano et al.~\cite{romano2020classification} proposed Adaptive Prediction Sets (APS), which compute nonconformity by accumulating sorted class probabilities until the true label is included. Angelopoulos et al.~\cite{angelopoulos2021uncertainty} further introduced Regularized Adaptive Prediction Sets (RAPS), which regularize the tail of the sorted probability vector to reduce sensitivity to noisy low-probability classes and improve set efficiency for image classifiers. Recent conformal prediction methods for zero-shot vision-language models also largely follow this probability-vector paradigm, evaluating LAC, APS, and RAPS on softmax-normalized class scores~\cite{silvarodriguez2025conformal}. Despite their success in closed-set classification and zero-shot model transfer, these scores are still defined with respect to a fixed candidate label set. In open-world semantic safety, however, new safety-relevant labels may be discovered and inserted into the dictionary during deployment, which changes the probability normalization and may alter the ranking of existing labels. Therefore, there remains a lack of nonconformity score design tailored to uncertainty quantification under an open-world setting with an online-expanding label dictionary.

\textbf{Conformal uncertainty quantification for safety-critical planning and control.} Recent work has extensively applied CP to control synthesis: Lindemann et al.\ \cite{lindemann2023safe} used CP to construct safe reachable sets for dynamic obstacle avoidance; Yang et al.\ \cite{yang2023safe} extended this to perception-based control under stochastic sensor uncertainty; 
Dixit et al.\ \cite{dixit2023adaptive} and Wang et al.\ \cite{WANG2026101751,wang2025sparcpredictionbasedsafecontrol} used adaptive conformal prediction and robust conformal prediction for safe motion planning among dynamic agents; and \cite{yu2023,stamouli2024recursively,sun2024conformal} further explored CP-based guarantees in temporal logic planning, shrinking-horizon MPC, and diffusion-based dynamics models, respectively. 
CP has also been applied to robot and autonomous-driving perception~\citep{xu2025openvocabseg, doula2025krps}. Uncertainty sets in these works are typically defined over fixed prediction targets, such as states, trajectories, or perception errors. \cite{sundarsingh2026safe} proposed the concept of \textit{semantic safety}---maintaining semantically appropriate distances from objects based on their category---and used conformalized semantic maps to provide probabilistic safety guarantees. However, their approach operates on a \textit{closed} semantic map with a fixed label set and relies on offline calibration. 
Several studies have applied conformal prediction to quantify the uncertainty of vision-language models \cite{tu2024empirical, kostumov2024uncertaintyawareevaluationvisionlanguagemodels}, revealing that VLM confidence scores are often poorly calibrated---model uncertainty is not aligned with prediction accuracy. However, these analyses are limited to static image classification or visual question answering (VQA) benchmarks and do not address how such uncertainty affects closed-loop planning in embodied agents.

\section{System Dynamics and Optimization Formulation}
\label{app:dynamics}

The robot's motion is modeled as a discrete-time system:
\begin{equation}
    x_{t+1} = f(x_t, u_t),
\end{equation}
where $t \in \{0, \dots, H\}$ denotes the time step, $x_t \in \mathcal{X} \subset \mathbb{R}^{n_x}$ represents the robot state (e.g., position and orientation), and $u_t \in \mathcal{U} \subset \mathbb{R}^{n_u}$ denotes the control input constrained by a compact set $\mathcal{U}$. The full chance-constrained optimal control problem is:

\begin{myprob}
Given a target miscoverage rate $\alpha \in (0, 1)$, compute a control sequence that minimizes the path length while satisfying the semantic safety constraint:
\begin{subequations}
\label{eq:optimization_problem}
\begin{align}
        \min_{\mathbf{u}_{0:H-1}} \quad & J(\mathbf{x}_{0:H-1},\mathbf{u}_{0:H-1}) \\
    \text{s.t.} \quad & x_{t+1} = f(x_t, u_t), \quad \forall t, \\
    & x_0 = x_{\text{init}}, \quad x_H \in \mathcal{X}_{\text{goal}}, \\
    & u_t \in \mathcal{U}, \quad \forall t, \\
    & \mathbb{P}\!\left( \bigwedge_{j \in \mathcal{N}_r(x_t)} \left( \| x_t - c_j \|_2 \ge D_{\text{safe}}(l_j^{(t)}) \right) \right) \ge 1 - \alpha, \;\; \forall t.
\end{align}
\end{subequations}
\end{myprob}

\section{Proofs}
\label{app:proofs}

\subsection{Proof of Proposition~\ref{prop:coverage_delay} (Asymptotic Calibration under Delay)}
\label{app:proof_prop1}

The proof proceeds in two steps: we first establish that $\alpha_t$ remains bounded, then use a telescoping argument.

\begin{mylem}[\textbf{Boundedness of $\alpha_t$}]
\label{lem:bound_alpha}
Under the ACI update $\alpha_{t+1} = \alpha_t + \gamma(\alpha - \emph{err}_t)$ with $\alpha_1 \in [0,1]$ and $\gamma > 0$, we have $\alpha_t \in [-\gamma, 1+\gamma]$ for all $t \ge 1$, almost surely.
\end{mylem}

\begin{proof}
We show $\alpha_t \ge -\gamma$ for all $t$; the upper bound $\alpha_t \le 1+\gamma$ follows by a symmetric argument.

Suppose for contradiction that $\mathbb{P}(\exists\, t : \alpha_t < -\gamma) > 0$. Since each step changes $\alpha_t$ by at most $\gamma$ (because $|\alpha - \text{err}_t| \le 1$), for $\alpha_t$ to drop below $-\gamma$ it must first pass through a value $\alpha_t \in (-\gamma, 0)$. Consider any such $t$ with $\alpha_t < 0$:
\begin{enumerate}
    \item[(i)] Since $1 - \alpha_t > 1$, the empirical quantile satisfies $Q_t = \hat{Q}(1 - \alpha_t) = +\infty$.
    \item[(ii)] The nonconformity score $S_{t-\tau}$ is finite (since $s(\cdot,\cdot) \in [0,2]$ by the Cauchy--Schwarz bound on cosine similarity).
    \item[(iii)] Therefore $S_{t-\tau} \le Q_t = +\infty$, which gives $\text{err}_t = \mathbb{I}(S_{t-\tau} > Q_t) = 0$.
    \item[(iv)] The update becomes $\alpha_{t+1} = \alpha_t + \gamma\alpha \ge \alpha_t$ (since $\alpha > 0$).
\end{enumerate}
This contradicts the assumption that $\alpha_t$ can decrease when $\alpha_t < 0$, so $\alpha_t$ can never drop below $-\gamma$.

For the upper bound: if $\alpha_t > 1$, then $1 - \alpha_t < 0$, so $Q_t = \hat{Q}(1-\alpha_t) = -\infty$ (the quantile at a negative level). Every finite score exceeds $-\infty$, giving $\text{err}_t = 1$ and $\alpha_{t+1} = \alpha_t - \gamma(1-\alpha) \le \alpha_t$.
\end{proof}

\noindent\textit{Proof of Proposition~\ref{prop:coverage_delay}.}\;
Expanding the ACI recursion from $t=1$ to $T$:
\begin{equation}
    \alpha_{T+1} = \alpha_1 + \gamma \sum_{t=1}^T (\alpha - \text{err}_t).
    \label{eq:telescope}
\end{equation}
By Lemma~\ref{lem:bound_alpha}, the left-hand side satisfies $\alpha_{T+1} \in [-\gamma, 1+\gamma]$. Rearranging~\eqref{eq:telescope}:
\begin{equation}
    \sum_{t=1}^T (\alpha - \text{err}_t) = \frac{\alpha_{T+1} - \alpha_1}{\gamma}.
\end{equation}
Taking absolute values and dividing by $T$:
\begin{equation}
    \left| \frac{1}{T}\sum_{t=1}^T \text{err}_t - \alpha \right| = \frac{|\alpha_{T+1} - \alpha_1|}{T\gamma} \le \frac{\max\{\alpha_1, 1-\alpha_1\} + \gamma}{T\gamma},
\end{equation}
where the inequality uses $|\alpha_{T+1}| \le 1 + \gamma$ and $|\alpha_1| \le 1$. The bound is $O(1/T)$ and contains no dependence on $\tau$: the delay affects \emph{which} neighborhood is evaluated but not the algebraic structure of the recursion. \hfill$\square$

\subsection{Proof of Proposition~\ref{prop:tracking} (Planning-Time Coverage under Delay)}
\label{app:proof_prop2}

The proof has three steps: (C1)~a martingale argument linking realized errors to conditional probabilities, (C2)~bounding the conditional--marginal gap via Assumption \textbf{A2}, and (C3)~bounding the calibration--planning distribution gap via Assumption \textbf{A1}.

\begin{mylem}[\textbf{C1: Martingale Convergence}]
\label{lem:martingale}
Let $\tilde{e}_t \triangleq \mathbb{P}(\emph{err}_t = 1 \mid \mathcal{F}_{t-1})$. Then $\frac{1}{T}\sum_{t=1}^T (\emph{err}_t - \tilde{e}_t) \to 0$ almost surely.
\end{mylem}

\begin{proof}
Define $d_t \triangleq \text{err}_t - \tilde{e}_t$. By construction, $\{d_t\}$ is a martingale difference sequence adapted to $\mathcal{F}_t = \sigma(\text{err}_1, \dots, \text{err}_t)$:
\begin{equation}
    \mathbb{E}[d_t \mid \mathcal{F}_{t-1}] = \mathbb{E}[\text{err}_t \mid \mathcal{F}_{t-1}] - \tilde{e}_t = 0.
\end{equation}
Since $\text{err}_t \in \{0,1\}$ and $\tilde{e}_t \in [0,1]$, we have $|d_t| \le 1$ and:
\begin{equation}
    \text{Var}(d_t \mid \mathcal{F}_{t-1}) = \mathbb{E}[d_t^2 \mid \mathcal{F}_{t-1}] \le \mathbb{E}[\text{err}_t^2 \mid \mathcal{F}_{t-1}] = \tilde{e}_t \le 1.
\end{equation}
More tightly, $\text{Var}(d_t \mid \mathcal{F}_{t-1}) = \tilde{e}_t(1-\tilde{e}_t) \le 1/4$. Since $\sum_{t=1}^{\infty} \text{Var}(d_t \mid \mathcal{F}_{t-1})/t^2 \le \sum_{t=1}^{\infty} 1/(4t^2) < \infty$, the strong law of large numbers for martingale differences (see, e.g., \citet[Theorem~2.18]{hall2014martingale}) gives $\frac{1}{T}\sum_{t=1}^T d_t \to 0$ a.s.
\end{proof}

\begin{mylem}[\textbf{C2: Conditional--Marginal Gap}]
\label{lem:cond_marginal}
Define $\bar{e}_t \triangleq 1 - F_{t-\tau}(Q_t)$, the marginal miscoverage probability at the calibration neighborhood. By the first condition of Assumption \textbf{A2}:
\begin{equation}
    \limsup_{T \to \infty} \left| \frac{1}{T}\sum_{t=1}^T \bar{e}_t - \alpha \right| \le \epsilon.
\end{equation}
\end{mylem}

\begin{proof}
By Assumption \textbf{A2}, $|\tilde{e}_t - \bar{e}_t| \le \epsilon$ for all $t$. Therefore:
\begin{equation}
    \left| \frac{1}{T}\sum_{t=1}^T \tilde{e}_t - \frac{1}{T}\sum_{t=1}^T \bar{e}_t \right| \le \frac{1}{T}\sum_{t=1}^T |\tilde{e}_t - \bar{e}_t| \le \epsilon.
    \label{eq:cond_marginal_gap}
\end{equation}
From Lemma~\ref{lem:martingale} and Proposition~\ref{prop:coverage_delay}:
\begin{equation}
    \frac{1}{T}\sum_{t=1}^T \tilde{e}_t = \frac{1}{T}\sum_{t=1}^T \text{err}_t - \frac{1}{T}\sum_{t=1}^T d_t \;\xrightarrow{T \to \infty}\; \alpha \quad \text{a.s.}
    \label{eq:conditional_convergence}
\end{equation}
Combining~\eqref{eq:cond_marginal_gap} and~\eqref{eq:conditional_convergence} via the triangle inequality gives the result.
\end{proof}

\begin{mylem}[\textbf{C3: Calibration--Planning Distribution Gap}]
\label{lem:calib_plan_gap}
Under Assumption \textbf{A1}, the gap between the planning-time and calibration-time marginal miscoverage satisfies:
\begin{equation}
    \left| (1 - F_t(Q_t)) - \bar{e}_t \right| \le \tau \Delta, \quad \forall\, t \ge 1.
\end{equation}
\end{mylem}

\begin{proof}
By definition, $\bar{e}_t = 1 - F_{t-\tau}(Q_t)$ and the planning-time miscoverage is $1 - F_t(Q_t)$. Applying Assumption \textbf{A1} iteratively via the triangle inequality:
\begin{align}
    |F_{t-\tau}(Q_t) - F_t(Q_t)| &\le \sum_{k=0}^{\tau-1} |F_{t-\tau+k}(Q_t) - F_{t-\tau+k+1}(Q_t)| \le \tau \Delta.
    \label{eq:cdf_triangle}
\end{align}
Since $(1 - F_t(Q_t)) - \bar{e}_t = F_{t-\tau}(Q_t) - F_t(Q_t)$, the result follows.
\end{proof}

\noindent\textit{Proof of Proposition~\ref{prop:tracking}.}\;
Combining Lemmas~\ref{lem:cond_marginal} and~\ref{lem:calib_plan_gap} via the triangle inequality:
\begin{align}
    &\limsup_{T \to \infty} \left| \frac{1}{T}\sum_{t=1}^T (1 - F_t(Q_t)) - \alpha \right| \notag \\
    &\le \limsup_{T \to \infty} \left| \frac{1}{T}\sum_{t=1}^T (1 - F_t(Q_t)) - \frac{1}{T}\sum_{t=1}^T \bar{e}_t \right| + \limsup_{T \to \infty} \left| \frac{1}{T}\sum_{t=1}^T \bar{e}_t - \alpha \right| \notag \\
    &\le \tau\Delta + \epsilon.
\end{align}
The first term uses Lemma~\ref{lem:calib_plan_gap} (each summand is bounded by $\tau\Delta$, so the average is too). The second term is Lemma~\ref{lem:cond_marginal}. \hfill$\square$

\subsection{Proof of Theorem~\ref{thm:safety} (Per-Step Semantic Safety)}
\label{app:proof_thm}

We first establish the key equivalence used throughout the safety analysis.

\begin{mylem}[\textbf{Neighborhood Coverage Equivalence}]
\label{lem:joint_coverage_app}
$S_t \le Q_t$ if and only if $l_j^{(t)} \in \mathcal{C}_j^{(t)}$ for all $j \in \mathcal{N}_r(x_t)$.
\end{mylem}

\begin{proof}
By definition, $S_t = \max_{j \in \mathcal{N}_r(x_t)} s(\mathbf{v}_j^{(t)}, \mathbf{v}_{\text{cal},j}^{(t)})$. Hence:
\begin{align}
    S_t \le Q_t &\iff s(\mathbf{v}_j^{(t)}, \mathbf{v}_{\text{cal},j}^{(t)}) \le Q_t \quad \forall\, j \in \mathcal{N}_r(x_t) \\
    &\iff l_j^{(t)} \in \{l \in \mathcal{V}_t : s(\mathbf{v}_j^{(t)}, \psi(l)) \le Q_t\} = \mathcal{C}_j^{(t)} \quad \forall\, j \in \mathcal{N}_r(x_t). \qedhere
\end{align}
\end{proof}

\noindent\textit{Proof of Theorem~\ref{thm:safety}.}\;
Let $E = \{S_t \le Q_t\}$, so $\mathbb{P}(E) \ge 1 - \alpha_t'$ by condition~(i). We show that $E$ implies no safety violation.

\textit{Step 1: Coverage implies worst-case dominance.}\;
By Lemma~\ref{lem:joint_coverage_app}, event $E$ implies $l_j^{(t)} \in \mathcal{C}_j^{(t)}$ for all $j \in \mathcal{N}_r(x_t)$. For any such $j$:
\begin{equation}
    D_{\text{safe}}(l_j^{(t)}) \le \max_{l \in \mathcal{C}_j^{(t)}} D_{\text{safe}}(l) = d_{\text{wc}}^{(j,t)}.
    \label{eq:dsafe_dwc}
\end{equation}

\textit{Step 2: Worst-case planning ensures safety.}\;
Consider any $j \in \mathcal{N}_r(x_t)$ with $D_{\text{safe}}(l_j^{(t)}) > 0$. By reliable obstacle detection, $j \in \mathcal{J}_{\text{obs}}^{(t)}$. Condition~(ii) gives $\|x_t - c_j\| \ge d_{\text{wc}}^{(j,t)}$. Combining with~\eqref{eq:dsafe_dwc}:
\begin{equation}
    \|x_t - c_j\| \ge d_{\text{wc}}^{(j,t)} \ge D_{\text{safe}}(l_j^{(t)}).
\end{equation}
For cells with $D_{\text{safe}}(l_j^{(t)}) = 0$, the safety constraint is trivially satisfied.

\textit{Step 3: Probability bound.}\;
The violation event $V = \{\exists\, j \in \mathcal{N}_r(x_t) : \|x_t - c_j\| < D_{\text{safe}}(l_j^{(t)})\}$ satisfies $V \subseteq \neg E$ by Steps~1--2. Therefore $\mathbb{P}(V) \le \mathbb{P}(\neg E) = 1 - \mathbb{P}(E) \le \alpha_t'$. \hfill$\square$

\subsection{Proof of Corollary~\ref{cor:mission} (Long-Run Average Safety)}
\label{app:proof_cor}

The proof connects Theorem~\ref{thm:safety} (per-step) with Proposition~\ref{prop:tracking} (long-run) via two intermediate results.

\begin{mylem}[\textbf{Per-Step Violation Bound}]
\label{lem:per_step_violation}
Whenever the planner enforces $\|x_t - c_j\| \ge d_{\emph{wc}}^{(j,t)}$ for all $j \in \mathcal{J}_{\emph{obs}}^{(t)} \cap \mathcal{N}_r(x_t)$:
\begin{equation}
    \mathbb{P}(V_t \mid Q_t) \le 1 - F_t(Q_t) + \epsilon,
\end{equation}
where $V_t = \{\exists\, j \in \mathcal{N}_r(x_t) : \|x_t - c_j\| < D_{\emph{safe}}(l_j^{(t)})\}$.
\end{mylem}

\begin{proof}
By the argument in Theorem~\ref{thm:safety}, $V_t \subseteq \{S_t > Q_t\}$. Since $Q_t$ is $\mathcal{F}_{t-1}$-measurable (it depends only on past errors):
\begin{equation}
    \mathbb{P}(V_t \mid Q_t) \le \mathbb{P}(S_t > Q_t \mid Q_t).
\end{equation}
By the tower property of conditional expectation:
\begin{equation}
    \mathbb{P}(S_t > Q_t \mid Q_t) = \mathbb{E}\!\left[\mathbb{P}(S_t > Q_t \mid \mathcal{F}_{t-1}) \;\middle|\; Q_t\right].
\end{equation}
Applying the second condition of Assumption \textbf{A2}:
\begin{equation}
    \mathbb{P}(S_t > Q_t \mid \mathcal{F}_{t-1}) \le (1 - F_t(Q_t)) + \epsilon.
\end{equation}
Since the right-hand side is $\mathcal{F}_{t-1}$-measurable, the tower property preserves the bound:
\begin{equation}
    \mathbb{P}(V_t \mid Q_t) \le (1 - F_t(Q_t)) + \epsilon. \qedhere
\end{equation}
\end{proof}

\noindent\textit{Proof of Corollary~\ref{cor:mission}.}\;
Averaging Lemma~\ref{lem:per_step_violation} over $t = 1, \dots, T$:
\begin{equation}
    \frac{1}{T}\sum_{t=1}^{T} \mathbb{P}(V_t \mid Q_t) \le \frac{1}{T}\sum_{t=1}^{T} \bigl(1 - F_t(Q_t)\bigr) + \epsilon.
\end{equation}
Taking $\limsup$ and applying Proposition~\ref{prop:tracking}:
\begin{align}
    \limsup_{T \to \infty}\; \frac{1}{T}\sum_{t=1}^{T} \mathbb{P}(V_t \mid Q_t) &\le \limsup_{T \to \infty}\; \frac{1}{T}\sum_{t=1}^{T} \bigl(1 - F_t(Q_t)\bigr) + \epsilon \\
    &\le (\alpha + \tau\Delta + \epsilon) + \epsilon = \alpha + \tau\Delta + 2\epsilon.
\end{align}
The factor of $2\epsilon$ arises from two independent applications of Assumption \textbf{A2}: once in Proposition~\ref{prop:tracking} (calibration-side, via Lemma~\ref{lem:cond_marginal}) and once in Lemma~\ref{lem:per_step_violation} (planning-side). \hfill$\square$

\section{Algorithm}
\label{app:algorithm}

\begin{algorithm}[H]
\caption{CoSMO: Conformal Semantic Mapping for Open-World Safe Robot Navigation}
\label{alg:cosmo}
\small
\hspace*{0.02in} {\bf Input:}
Feature map $\{\mathbf{v}_j^{(t)}\}$; Robot state $x_t$; Dictionary $(\mathcal{V}_t, D_{\text{safe}})$; Threshold $Q_t$; Risk $\alpha_t$; Buffer $\mathcal{S}$; Step size $\gamma$; Target $\alpha$; Radius $r$\\
\hspace*{0.02in} {\bf Output:}
Updated $(\mathcal{V}_{t+1}, Q_{t+1}, \alpha_{t+1})$; Worst-case safety map
\begin{algorithmic}[1]
\State \textit{// Phase 1: Expert oracle feedback \& dictionary expansion}
\State Expert reveals GT labels $\{l_j^{(t-\tau)}\}_{j \in \mathcal{N}_r(x_{t-\tau})}$ for neighborhood at $t - \tau$
\For{each $j \in \mathcal{N}_r(x_{t-\tau})$}
    \If{$l_j^{(t-\tau)} \notin \mathcal{V}_t$} \Comment{Open-vocabulary discovery}
        \State Compute $\mathbf{v}_{\text{cal},j}^{(t-\tau)} \gets \psi(l_j^{(t-\tau)})$; assign $D_{\text{safe}}(l_j^{(t-\tau)})$ from expert
        \State $\mathcal{V}_t \gets \mathcal{V}_t \cup \{l_j^{(t-\tau)}\}$ \Comment{Monotonic expansion}
    \EndIf
\EndFor
\State \textit{// ACI calibration with neighborhood score}
\State $S \gets \max_{j \in \mathcal{N}_r(x_{t-\tau})} s(\mathbf{v}_j^{(t-\tau)}, \mathbf{v}_{\text{cal},j}^{(t-\tau)})$ \Comment{Neighborhood NCS}
\State $\text{err} \gets \mathbb{I}(S > Q_t)$
\State $\alpha_t \gets \alpha_t + \gamma(\alpha - \text{err})$
\State $\mathcal{S} \gets \mathcal{S} \cup \{S\}$;\; $Q_t \gets \hat{Q}(\mathcal{S},\; 1 - \alpha_t)$
\vspace{0.05in}
\State \textit{// Phase 2: Coverage sets \& worst-case safety}
\For{each obstacle cell $j \in \mathcal{J}_{\text{obs}} \cap \mathcal{N}_r(x_t)$}
    \State $\mathcal{C}_j \gets \{l \in \mathcal{V}_t \mid s(\mathbf{v}_j,\psi(l)) \le Q_t\}$ \Comment{Coverage set}
    \If{$\mathcal{C}_j = \emptyset$}
        \State $d_{\text{wc}}^{(j)} \gets 0$ \Comment{No semantic dilation before discovery}
    \Else
        \State $d_{\text{wc}}^{(j)} \gets \max_{l \in \mathcal{C}_j} D_{\text{safe}}(l)$ \Comment{Worst-case radius}
    \EndIf
    \State Dilate cell $j$ by $d_{\text{wc}}^{(j)}$ on obstacle map
\EndFor
\end{algorithmic}
\end{algorithm}

\section{Obstacle Safety Dictionary}
\label{app:obstacle_dict}

Because semantic safety distances are not standardized in the literature and the choice of dictionary is a deliberate modeling decision, we consider four dictionary variants for the 21 MP3D goal categories~\citep{chang2017matterport3d}: V1--V3 for dictionary ablation and V1/V4 for safety-distance sensitivity analysis. The grid resolution is $0.1$\,m/cell, so a radius of 5 cells corresponds to $0.5$\,m physical clearance.
\begin{itemize}
\item \textbf{V1 (5-hazard, default):} 5 categories that frequently appear across the dataset scenes---\texttt{table}, \texttt{cabinet}, \texttt{chest\_of\_drawers}, \texttt{shower}, \texttt{tv\_monitor}---receive nonzero safety distances with heterogeneous radii reflecting their size and risk; the remaining 16 categories have $D_{\text{safe}} = 0$. \textbf{This is the dictionary used in all main experiments unless otherwise noted.}
\item \textbf{V2 (21-hazard, mixed):} all 21 categories are treated as hazards; the 5 representative categories keep their V1 radii and the other 16 each receive a small radius of 2 cells.
\item \textbf{V3 (21-hazard, uniform):} all 21 categories are treated as hazards with a uniform radius of 5 cells, matching the maximum radius among the V1 hazards.
\item \textbf{V4 (5-hazard, low-frequency):} the five least frequent categories---\texttt{bathtub}, \texttt{clothes}, \texttt{gym\_equipment}, \texttt{seating}, and \texttt{fireplace}---each receive a nominal safety distance of 4 cells ($0.4$\,m); the remaining 16 categories have $D_{\text{safe}}=0$. V4 is used in the sensitivity analysis below.
\end{itemize}

\begin{table}[H]
\caption{Obstacle safety dictionaries (top: per-category safety distances $D_{\text{safe}}$ in grid cells, $0.1$\,m/cell) and results on 330 MP3D \texttt{val\_ablation} episodes (bottom: V1--V3 dictionary ablation and the unperturbed V4 sensitivity setting, with CoSMO using $\alpha=0.1$ and $\gamma=0.05$). \textbf{Bold} entries: 5 representative hazard categories shared by V1 and V2 (top).}
\label{tab:obstacle_dict}
\centering
\small
    \setlength{\tabcolsep}{3pt}
    \vspace{1em}
\begin{tabular}{@{}lcccc|lcccc@{}}
\toprule
\textbf{Category} & V1 & V2 & V3 & V4 & \textbf{Category} & V1 & V2 & V3 & V4 \\
\midrule
chair                       & 0 & 2 & 5 & 0 & stool                       & 0 & 2 & 5 & 0 \\
\textbf{table}              & \textbf{5} & \textbf{5} & 5 & 0 & towel                       & 0 & 2 & 5 & 0 \\
picture                     & 0 & 2 & 5 & 0 & \textbf{tv\_monitor}        & \textbf{2} & \textbf{2} & 5 & 0 \\
\textbf{cabinet}            & \textbf{4} & \textbf{4} & 5 & 0 & \textbf{shower}             & \textbf{3} & \textbf{3} & 5 & 0 \\
cushion                     & 0 & 2 & 5 & 0 & counter                     & 0 & 2 & 5 & 0 \\
sofa                        & 0 & 2 & 5 & 0 & bathtub                     & 0 & 2 & 5 & 4 \\
bed                         & 0 & 2 & 5 & 0 & fireplace                   & 0 & 2 & 5 & 4 \\
\textbf{chest\_of\_drawers} & \textbf{4} & \textbf{4} & 5 & 0 & gym\_equipment              & 0 & 2 & 5 & 4 \\
plant                       & 0 & 2 & 5 & 0 & seating                     & 0 & 2 & 5 & 4 \\
sink                        & 0 & 2 & 5 & 0 & clothes                     & 0 & 2 & 5 & 4 \\
toilet                      & 0 & 2 & 5 & 0 &                             &   &   &   &   \\
\midrule
\multicolumn{10}{c}{\textit{Results on 330 MP3D \texttt{val\_ablation} episodes}} \\
\midrule
\multicolumn{2}{l}{\textbf{Metric}}         & \multicolumn{2}{c}{V1}              & \multicolumn{2}{c}{V2}     & \multicolumn{2}{c}{V3}     & \multicolumn{2}{c}{V4} \\
\midrule
\multicolumn{2}{l}{SR (\%) $\uparrow$}      & \multicolumn{2}{c}{12.7}   & \multicolumn{2}{c}{11.6}   & \multicolumn{2}{c}{9.1}   & \multicolumn{2}{c}{20.3} \\
\multicolumn{2}{l}{SPL $\uparrow$}          & \multicolumn{2}{c}{0.072}  & \multicolumn{2}{c}{0.063}  & \multicolumn{2}{c}{0.055} & \multicolumn{2}{c}{0.107} \\
\multicolumn{2}{l}{SCR (\%) $\downarrow$}   & \multicolumn{2}{c}{10.3}   & \multicolumn{2}{c}{10.1}   & \multicolumn{2}{c}{9.1}   & \multicolumn{2}{c}{4.8} \\
\multicolumn{2}{l}{STUCK (\%) $\downarrow$} & \multicolumn{2}{c}{43.9}   & \multicolumn{2}{c}{47.6}   & \multicolumn{2}{c}{57.3}  & \multicolumn{2}{c}{36.7} \\
\bottomrule
\end{tabular}

\smallskip
\vspace{1em}
\noindent The remaining 19 mpcat40 categories are structural or non-object classes excluded from semantic safety evaluation: \textit{void, wall, floor, door, window, curtain, stairs, ceiling, mirror, column, lighting, beam, railing, shelving, blinds, board\_panel, furniture, appliances, objects}.
\end{table}

\paragraph{Dictionary ablation.}
The bottom half of \Cref{tab:obstacle_dict} reports CoSMO's results under the V1--V4 setting. Across V1--V3, the semantic collision rate stays essentially flat ($9.1$--$10.3\%$). Success rate degrades, however, when uniform large radii are applied to all 21 categories (V3, SR $9.1\%$ vs.\ V1 $12.7\%$) because dilated objects block narrow passages. V1 attains the best SR/SPL among V1--V3 at a similar SCR level and remains the default dictionary. 

\subsection{Sensitivity to Safety Distances}
\label{app:dsafe_sensitivity}

Safety distances represent user-specified clearances. We assess sensitivity to their specification by perturbing the planning distances by $\pm1$ grid cell ($0.1$\,m) on the 330-episode ablation split, while evaluating SCR against the unchanged reference distances for each dictionary. We compare the default V1 dictionary with the low-frequency V4 dictionary defined above; its nominal 4-cell distances are perturbed to 3 or 5 cells.

\Cref{tab:dsafe_sensitivity} shows mild sensitivity for V1: underestimating the distances increases SCR from $10.3\%$ to $11.2\%$, while overestimating them reduces SCR to $10.0\%$ at a cost in SR and a higher STUCK rate. V4 is nearly unaffected by these perturbations, indicating that misspecifying infrequent categories has less influence on navigation performance.

\begin{table}[t]
    \centering
    \caption{Sensitivity to safety distance perturbations on 330 ablation episodes. Perturbations are in grid cells; SCR is evaluated using unchanged reference distances.}
    \label{tab:dsafe_sensitivity}
    \small
    \setlength{\tabcolsep}{5pt}
    \begin{tabular}{@{}lcrrrr@{}}
        \toprule
        Dict. & Perturbation & SR(\%) & SCR(\%) & SPL & STUCK(\%) \\
        \midrule
        V1 & $-1$ & 13.3 & 11.2 & 0.077 & 40.9 \\
        V1 & $0$  & 12.7 & 10.3 & 0.072 & 43.6 \\
        V1 & $+1$ & 11.5 & 10.0 & 0.068 & 47.3 \\
        \midrule
        V4 & $-1$ & 20.3 & 4.8 & 0.108 & 37.6 \\
        V4 & $0$  & 20.3 & 4.8 & 0.107 & 36.7 \\
        V4 & $+1$ & 20.0 & 5.2 & 0.103 & 36.7 \\
        \bottomrule
    \end{tabular}
\end{table}

\section{Implementation Details}
\label{app:implementation}

\Cref{tab:implementation} summarizes the models, software versions, and hyperparameters used in all experiments. All experiments were conducted on a workstation with 2$\times$ NVIDIA RTX 4090 GPUs (24\,GB VRAM each).
\begin{table}[H]
\caption{Models, software, and hyperparameters.}
\label{tab:implementation}
\centering
\small
\begin{tabular}{@{}ll@{}}
\toprule
\textbf{Component} & \textbf{Specification} \\
\midrule
\multicolumn{2}{@{}l}{\textit{Vision--Language Models}} \\
Navigation backbone & ConvNeXt-Large-d-320 (LAION-2B) \\
Obstacle detection (GCLIP) & ViT-B/16 (OpenAI CLIP) \\
Geometric detection & YOLO-World/L \\
Noisy expert ablation & Grounding-DINO (Swin-T) \\
\midrule
\multicolumn{2}{@{}l}{\textit{Simulator and Platform}} \\
Simulator & Habitat-Sim 0.2.2 \\
Dataset & Matterport3D (MP3D) \\
PyTorch & 2.4.1 (CUDA 12.1) \\
OpenCLIP & custom fork (SED) \\
Hardware & 2$\times$ NVIDIA RTX 4090 \\
\midrule
\multicolumn{2}{@{}l}{\textit{CoSMO Hyperparameters}} \\
Target miscoverage $\alpha$ & 0.1 \\
ACI step size $\gamma$ & 0.05 \\
Argmax cosine threshold $q_{\text{fixed}}$ & 0.30 \\
Grid resolution & 0.1\,m/cell \\
Max safety radius $\bar{D}$ & 5 cells (0.5\,m) \\
Neighborhood radius $r$ & 6 cells (0.6\,m) \\
\bottomrule
\end{tabular}
\end{table}

\section{Additional Experimental Results}
\label{app:experiments}

\subsection{Detailed Analysis of Main Results}
\label{app:analysis}

\textbf{conformal calibration reduces collisions.}\quad
Argmax assigns the safety radius of a single predicted label, whereas CoSMO considers all plausible labels within the calibrated threshold $Q_t$ and plans against the most conservative radius in the coverage set. This worst-case planning mechanism is the source of the $25.2\%$ relative SCR reduction from B to D. The SR cost reflects the inherent safety--efficiency tradeoff: wider coverage sets produce more conservative dilation, occasionally blocking possible paths.

\textbf{Uniform dilation is inefficient.}\quad
The ``Uniform dilation'' baseline applies $\bar{D}$ to \emph{every} detected obstacle regardless of its semantic label. It achieves comparable SCR to CoSMO ($9.1\%$ vs.\ $8.2\%$ at $80\%$ coverage) but at steep SR cost ($9.1\%$ vs.\ $12.4\%$). This is because most obstacles require smaller clearance than $\bar{D}$; CoSMO assigns label-specific radii, avoiding unnecessary over-dilation.

\textbf{More Known- vs.\ open-vocabulary SR difference.}\quad
Condition E (known-vocabulary) has lower SR than D ($10.4\%$ vs.\ $13.8\%$). This is because E applies semantic dilation from the first step, whereas D uses only geometric obstacle avoidance before labels are discovered, which is less restrictive and allows the robot to pass through narrow spaces early in the episode.

\subsection{Scene-dependent Safety--Success Tradeoff}
\label{app:scene_tradeoff}

\Cref{tab:scene_comparison} compares GCLIP Argmax and CoSMO on 200 episodes in each of two MP3D scenes, whose layouts are shown in \Cref{fig:scene_layouts}. In the more open scene \texttt{2azQ1b91cZZ}, CoSMO reduces SCR from $12.5\%$ to $7.5\%$ ($40\%$ relative reduction), while SR changes only from $14.5\%$ to $14.0\%$. In the more constrained scene \texttt{TbHJrupSAjP}, SCR decreases from $26.0\%$ to $23.0\%$, accompanied by a larger SR decrease from $24.5\%$ to $20.0\%$. The reported average STUCK rate is higher in the latter scene ($42.3\%$ vs.\ $26.0\%$). These comparisons suggest that alternative free-space routes can reduce the success-rate cost of semantic dilation.

\begin{table}[t]
    \centering
    \caption{Scene-level comparison on 200 episodes per scene. All rates are percentages.}
    \label{tab:scene_comparison}
    \small
    \setlength{\tabcolsep}{4pt}
    \begin{tabular}{@{}lcrrrrr@{}}
        \toprule
        & & \multicolumn{2}{c}{SR} & \multicolumn{2}{c}{SCR} & \\
        \cmidrule(lr){3-4}\cmidrule(lr){5-6}
        Scene & Episodes & Argmax & CoSMO & Argmax & CoSMO & Avg.STUCK \\
        \midrule
        \texttt{2azQ1b91cZZ} & 200 & 14.5 & 14.0 & 12.5 & 7.5 & 26.0 \\
        \texttt{TbHJrupSAjP} & 200 & 24.5 & 20.0 & 26.0 & 23.0 & 42.3 \\
        \bottomrule
    \end{tabular}
\end{table}

\begin{figure}[t]
    \centering
    \begin{subfigure}[c]{0.48\linewidth}
        \centering
        \includegraphics[width=\linewidth,height=0.18\textheight,keepaspectratio]{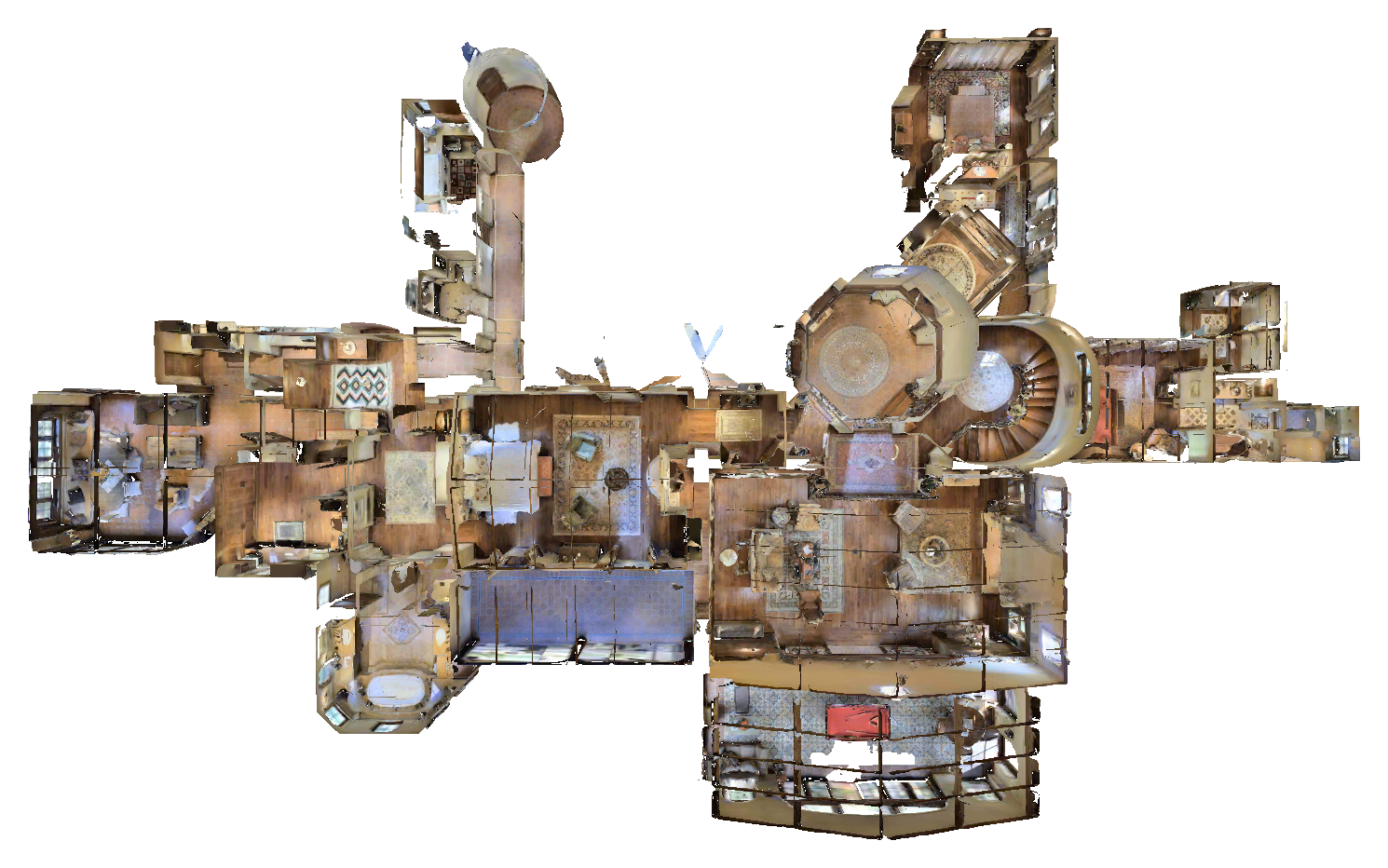}
        \caption{\texttt{2azQ1b91cZZ}}
    \end{subfigure}\hfill
    \begin{subfigure}[c]{0.48\linewidth}
        \centering
        \includegraphics[width=\linewidth,height=0.08\textheight,keepaspectratio]{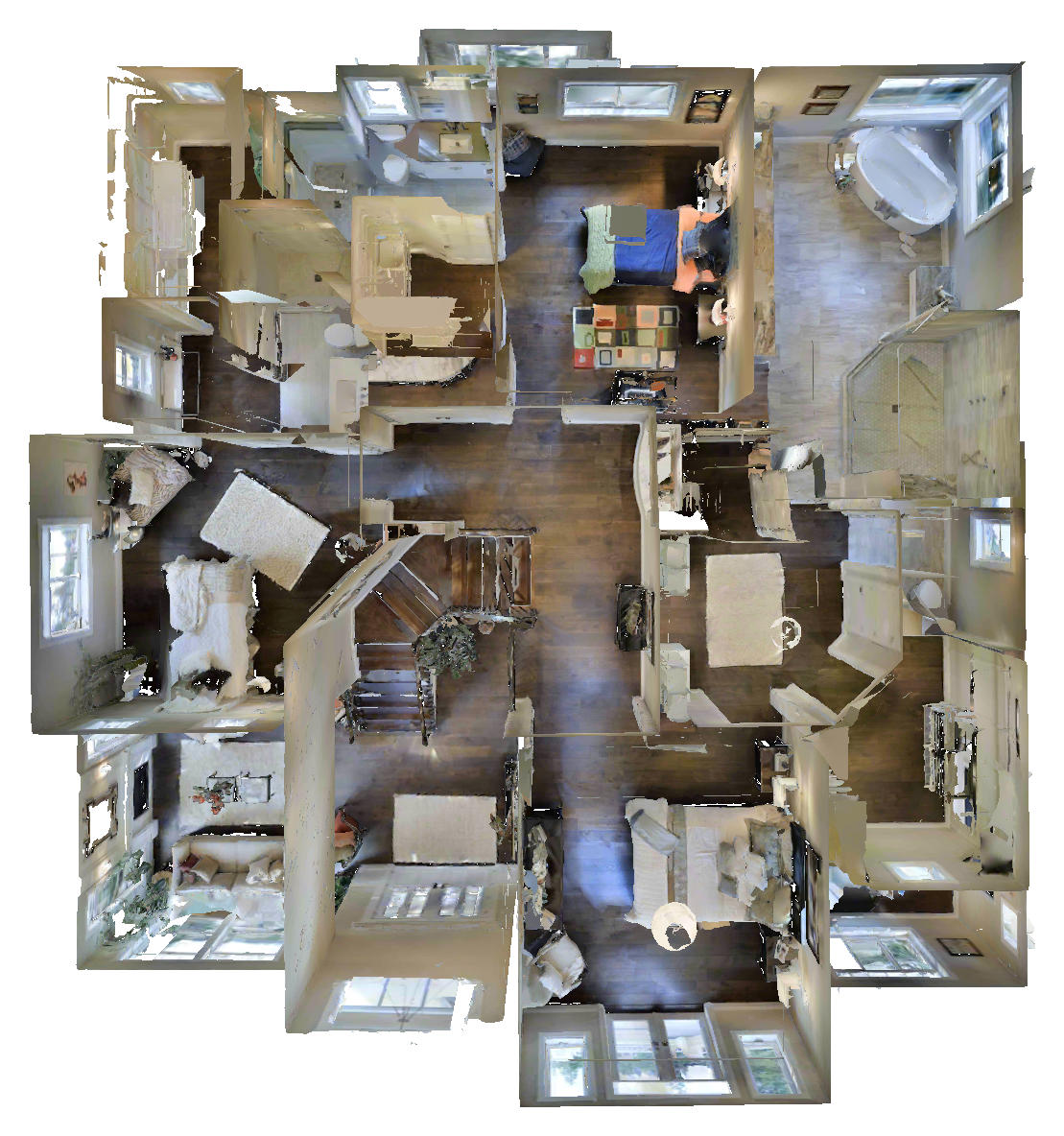}
        \caption{\texttt{TbHJrupSAjP}}
    \end{subfigure}
    \caption{Top-down views of the scenes in \Cref{tab:scene_comparison}: a more open layout (left) and a more constrained layout (right).}
    \label{fig:scene_layouts}
\end{figure}

\subsection{Nonconformity Score Function Ablation: Closed-Set vs.\ CoSMO}
\label{subsec:score_ablation}

\begin{wrapfigure}{r}{0.33\columnwidth}
    \centering
    \vspace{-1em}
    \includegraphics[width=0.33\columnwidth]{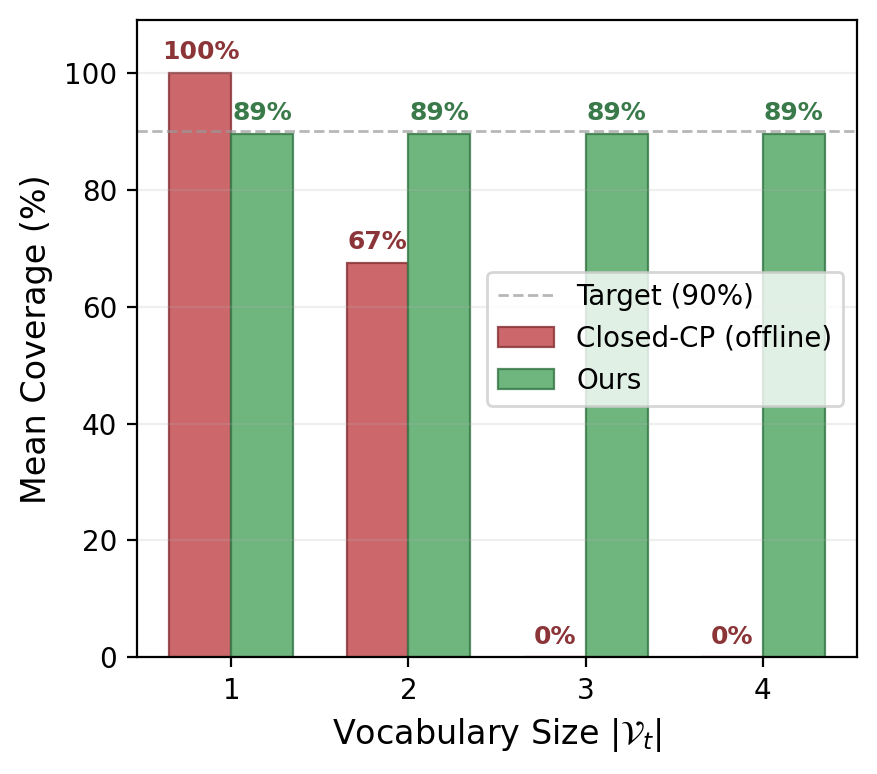}
    \caption{Empirical coverage vs.\ vocabulary size.}
    \label{fig:score_ablation}
    \vspace{-5em}
\end{wrapfigure}

We compare CoSMO against Closed-CP~\citep{sundarsingh2026safe}, which uses softmax-based nonconformity scores with an offline-calibrated threshold. \Cref{fig:score_ablation} shows empirical coverage as $|\mathcal{V}_t|$ grows from 1 to 4. Closed-CP degrades from $100\%$ to $0\%$ because each new label inflates the softmax denominator. CoSMO maintains $\approx 90\%$ coverage, as the pairwise cosine score is structurally independent of $|\mathcal{V}_t|$.

\subsection{Sensitivity to Expert Oracle}
\label{app:expert_quality}

\textbf{Expert delay.}\quad
\Cref{tab:delay_ablation} shows the effect of expert delay $\tau \in \{0, 5, 10, 20, 100\}$ on CoSMO. SCR stays within $10.0$--$10.9\%$ across all delay values, confirming that the asymptotic calibration guarantee (Proposition~\ref{prop:coverage_delay}) holds in practice. SR and SPL also show no systematic degradation with $\tau$.

\textbf{Expert accuracy.}\quad
\Cref{fig:calib_comparison} provides a detailed diagnostic comparison between the GT oracle and the GDINO oracle across the full calibration pipeline.

\begin{table}[H]
\caption{Delay ablation on 330 episodes. SCR and SPL show no systematic dependence on $\tau$, consistent with Proposition~\ref{prop:coverage_delay}.}
\label{tab:delay_ablation}
\centering
\begin{tabular}{lccc}
\toprule
$\tau$ & \textbf{SR} $\uparrow$ & \textbf{SCR} $\downarrow$ & \textbf{SPL} $\uparrow$ \\
\midrule
0   & 12.7\% & 10.3\% & 0.071 \\
5   & 11.8\% & 10.0\% & 0.067 \\
10  & 13.0\% & 10.3\% & 0.066 \\
20  & 14.2\% & 10.9\% & 0.076 \\
100 & 14.2\% & 10.0\% & 0.075 \\
\bottomrule
\end{tabular}
\end{table}

\begin{figure}[H]
\centering
\includegraphics[width=0.24\linewidth]{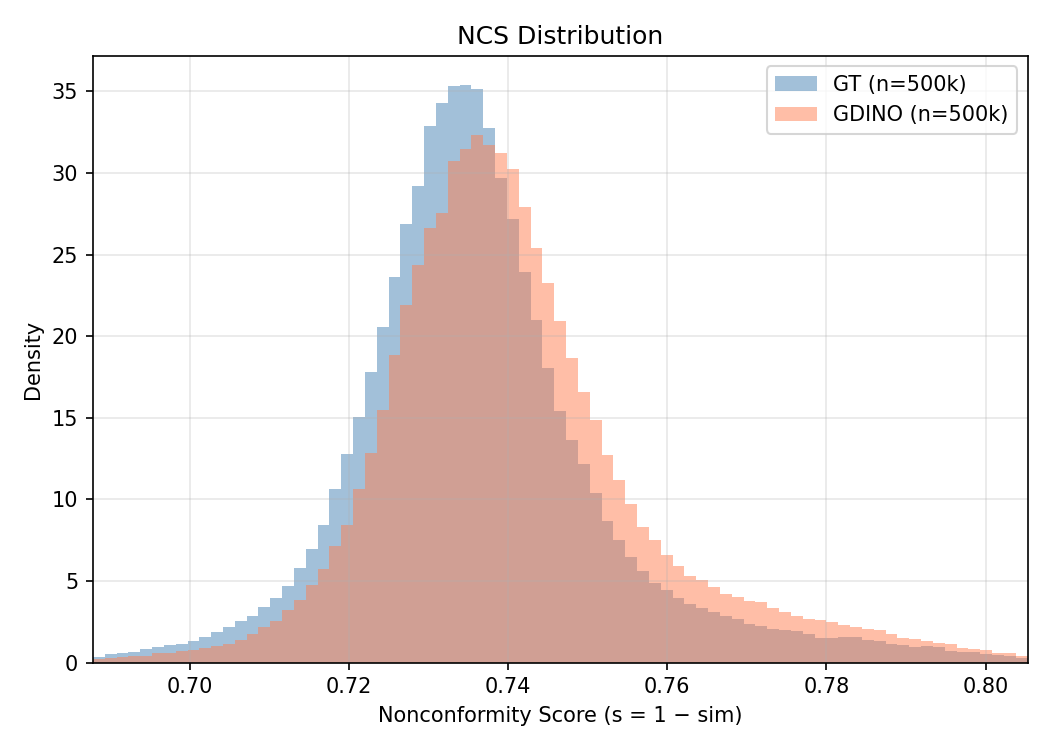}\hfill
\includegraphics[width=0.24\linewidth]{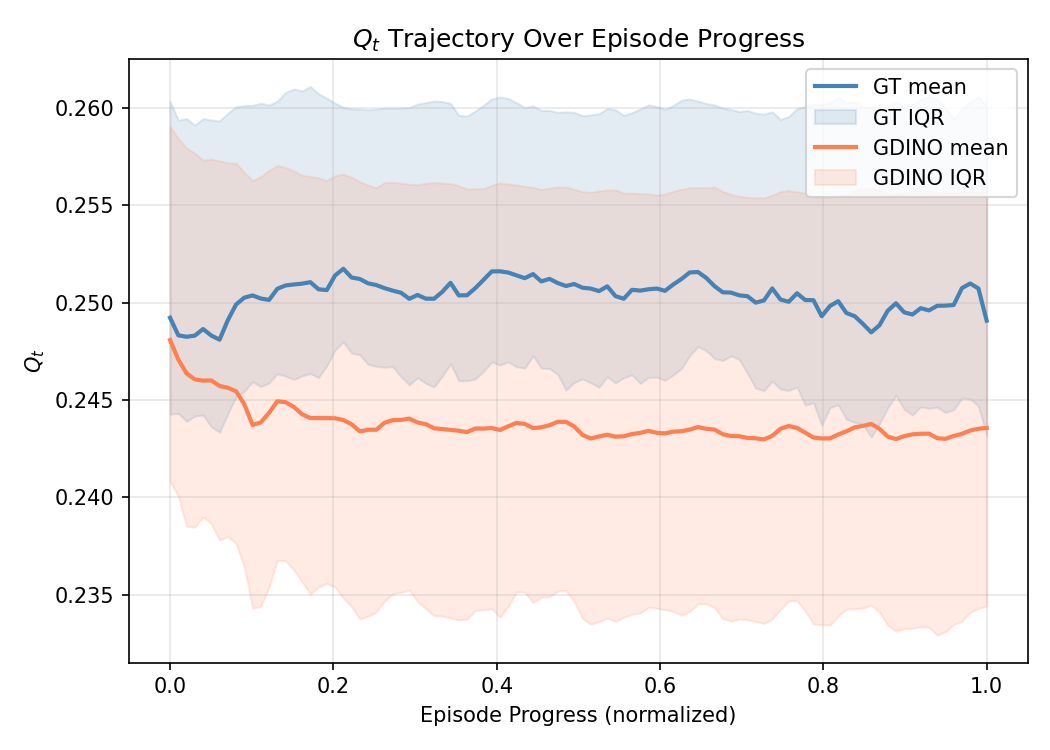}\hfill
\includegraphics[width=0.24\linewidth]{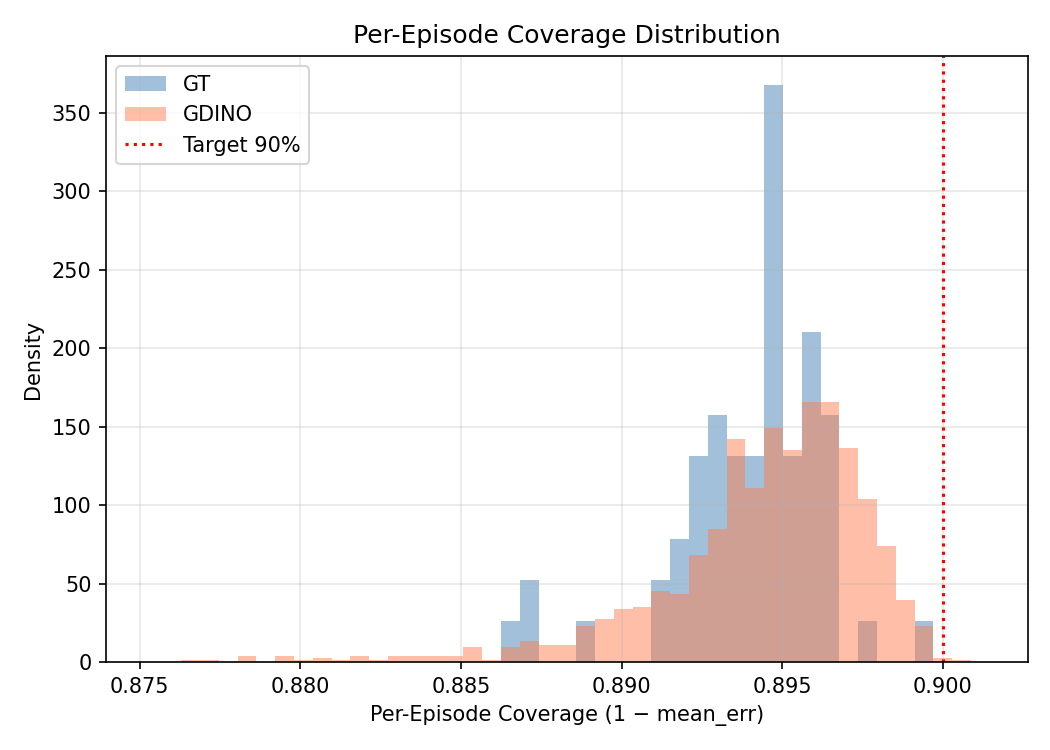}\hfill
\includegraphics[width=0.24\linewidth]{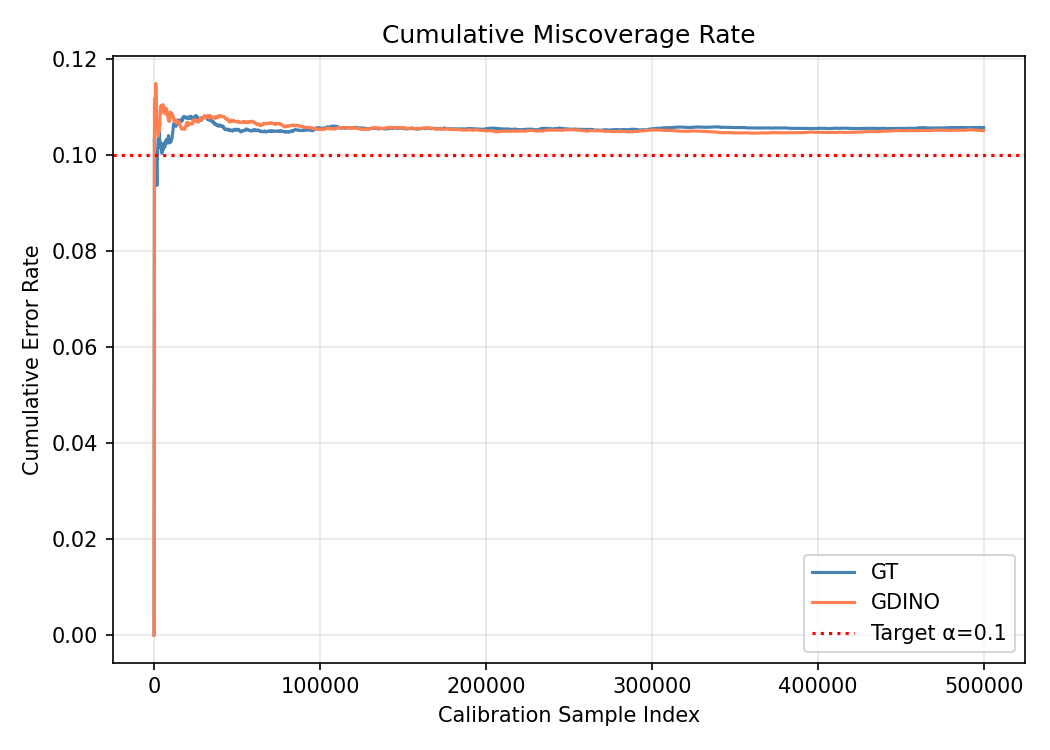}
\caption{GT oracle vs.\ Grounding-DINO oracle on nonconformity score distribution, $Q_t$ trajectory, per-episode coverage, and real-time miscoverage.} 
\label{fig:calib_comparison}
\end{figure}


Across all four diagnostics, the GT and GDINO oracles behave nearly identically: \textbf{(a)}~NCS distributions overlap closely; \textbf{(b)}~$Q_t$ trajectories both stabilize after early transients; \textbf{(c)}~per-episode coverage distributions match the $90\%$ target; and \textbf{(d)}~cumulative miscoverage converges to $\alpha = 0.1$ for both.
Despite the detector noise introduced by Grounding-DINO, the end-to-end coverage and error rates are virtually identical, confirming that Grounding-DINO is a viable drop-in replacement for the GT oracle.

\subsection{Coverage Target and Pareto Frontier}

\begin{table}[H]
\centering
\begin{minipage}[b]{0.51\columnwidth}
    \centering
    \setlength{\tabcolsep}{3pt}
    \renewcommand{\arraystretch}{0.95}
    \begin{tabular}{@{}llccc@{}}
    \toprule
    \textbf{Method} & \textbf{Cov.} & \textbf{SR}$\uparrow$ & \textbf{SCR}$\downarrow$ & \textbf{SPL}$\uparrow$ \\
    \midrule
    Baseline & --- & 8.5\% & 31.2\% & 0.053 \\
    Argmax   & --- & 15.9\% & 12.4\% & 0.062 \\
    Closed-CP~\citep{sundarsingh2026safe} & 80\% & 12.4\% & 11.2\% & 0.065 \\
    Uniform dilation & --- & 9.1\% & 9.1\% & 0.055\\
    \midrule
    Offline-CP & 20\% & 12.7\% & 12.4\% & 0.060 \\
    Offline-CP & 40\% & 13.9\% & 11.5\% & 0.073 \\
    Offline-CP & 60\% & 10.9\% & 10.9\% & 0.056 \\
    Offline-CP & 80\% & 12.1\% & 10.0\% & 0.065 \\
    \midrule
    CoSMO & 20\% & 12.1\% & 9.1\% & 0.059 \\
    CoSMO & 40\% & 12.7\% & 9.4\% & 0.068 \\
    CoSMO & 60\% & 12.1\% & 10.0\% & 0.067 \\
    CoSMO & 80\% & 12.4\% & 8.2\% & 0.064 \\
    \midrule
    CoSMO(cal.\ GDINO) & 80\% & 12.4\% & 13.0\% & 0.068 \\
    \bottomrule
    \end{tabular}
    \vspace{0.8em}
    \captionof{table}{Full results on 330 ablation episodes supporting \Cref{fig:pareto}. \textbf{Cov.} is the coverage target.}
    \label{tab:coverage_ablation}
\end{minipage}\hfill
\begin{minipage}[b]{0.46\columnwidth}
    \centering
    \includegraphics[width=\linewidth]{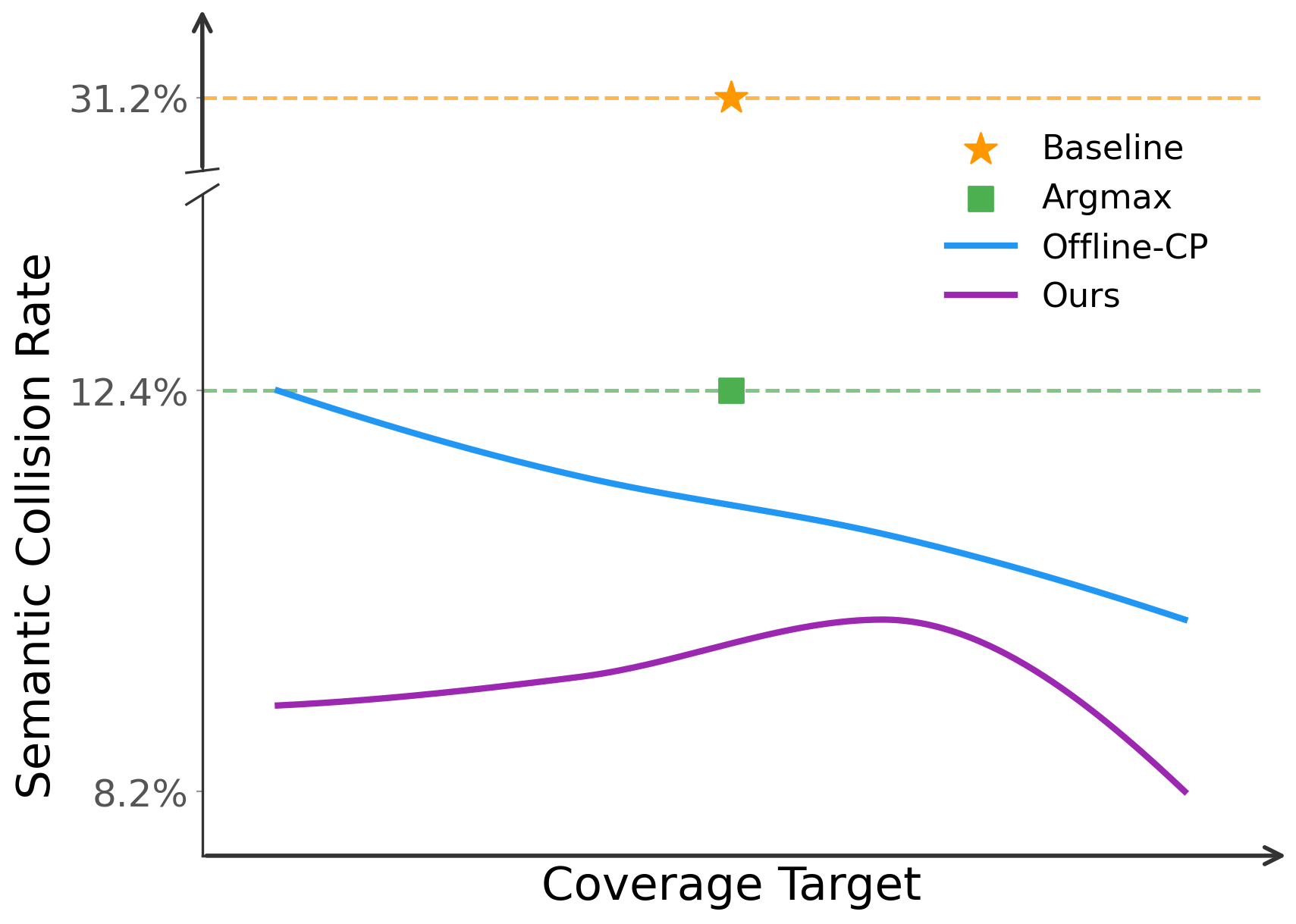}
    \vspace{0.3em}
    \captionof{figure}{SCR vs.\ coverage target on 330 ablation episodes. CoSMO achieves uniformly lower SCR than Offline-CP across all coverage targets. See \Cref{fig:pareto} for the corresponding Pareto frontier.}
    \label{fig:scr_coverage}
\end{minipage}
\end{table}

\Cref{tab:coverage_ablation} and \Cref{fig:scr_coverage} present the full coverage-target sweep on the 330-episode ablation split.
Two findings stand out.
\textbf{(i)~CoSMO dominates Offline-CP at every coverage target.} Across all four coverage levels ($20$--$80\%$), CoSMO achieves consistently lower SCR than Offline-CP while maintaining comparable SR and SPL. This confirms that online ACI calibration is strictly preferable to a fixed offline threshold, because it adapts to the non-stationary score distribution induced by dictionary expansion.
\textbf{(ii)~GDINO calibration incurs a modest safety cost.} CoSMO calibrated with Grounding-DINO at $80\%$ coverage reaches $13.0\%$ SCR, higher than its GT-calibrated counterpart ($8.2\%$). This gap is larger than on the main 2195-episode benchmark ($10.5\%$ vs.\ $9.8\%$, \Cref{tab:main_results}), likely because the smaller ablation split amplifies variance. Nevertheless, the GDINO variant still outperforms the Baseline ($31.2\%$) and Argmax ($12.4\%$) methods, confirming that a noisy but realistic detector remains a practical alternative to ground-truth supervision.
\subsection{Pre-discovery Safety}
\label{subsec:pre_discovery_safety}

Before a hazard category is discovered, its true label may be absent from the safety dictionary. For an empty coverage set $\mathcal{C}_j^{(t)}=\emptyset$, our default fallback is $d_{\mathrm{wc}}^{(j,t)}=0$, retaining geometric obstacle avoidance without additional semantic dilation. We compare this choice with conservative fallbacks of 2 and 5 cells on the 330-episode ablation split, with expert delays $\tau\in\{0,50\}$ (\Cref{tab:cold_start_fallback}). Across these settings, SCR ranges from $9.4\%$ to $12.1\%$, and neither safety nor success improves monotonically with the fallback radius. Only $0$--$6.45\%$ of semantic collisions occur before discovery, suggesting that the fallback has a limited empirical effect in this evaluation setting.

\begin{table}[t]
    \centering
    \caption{Cold-start fallback on 330 MP3D ablation episodes. Fallback radii are in grid cells ($0.1$\,m/cell). Pre-disc.\ and Post-disc.\ report semantic collision counts and their percentages of all semantic collisions in each setting.}
    \label{tab:cold_start_fallback}
    \small
    \setlength{\tabcolsep}{3pt}
    \begin{tabular}{@{}ccrrrrrr@{}}
        \toprule
        $\tau$ & Fallback $d_{\mathrm{wc}}$ & SR(\%) & SCR(\%) & SPL & STUCK(\%) & Pre-disc. & Post-disc. \\
        \midrule
        0  & 0 & 12.7 & 10.3 & 0.070 & 42.7 & 0 (0.0\%) & 34 (100\%) \\
        0  & 2 & 11.2 & 12.1 & 0.069 & 43.0 & 0 (0.0\%) & 40 (100\%) \\
        0  & 5 & 13.0 & 10.9 & 0.069 & 43.6 & 0 (0.0\%) & 36 (100\%) \\
        \midrule
        50 & 0 & 12.4 & 9.4 & 0.068 & 46.4 & 2 (6.45\%) & 29 (93.6\%) \\
        50 & 2 & 13.3 & 9.4 & 0.071 & 43.9 & 2 (6.45\%) & 29 (93.6\%) \\
        50 & 5 & 12.7 & 11.8 & 0.063 & 45.5 & 2 (5.13\%) & 37 (94.9\%) \\
        \bottomrule
    \end{tabular}
\end{table}

\paragraph{Theoretical implication.}
Choosing $d_{\mathrm{wc}}^{(j,t)}=\bar D$ for empty coverage sets, with $\bar D$ upper-bounding the safety distances of all categories including undiscovered ones, guarantees $d_{\mathrm{wc}}^{(j,t)}\ge D_{\mathrm{safe}}(l_j^{(t)})$ for those cells without requiring label discovery. The post-discovery assumption can therefore be relaxed to require $l_j^{(t)}\in\mathcal{V}_t$ only for cells with nonempty coverage sets. Under this weaker dictionary condition, $S_t\le Q_t$ ensures safety in both cases: empty sets are protected by the fallback, and nonempty sets by conformal coverage and worst-case dilation. Thus, the safety violation event still satisfies $V_t\subseteq\{S_t>Q_t\}$, preserving the bounds $\alpha_t'$ in Theorem~\ref{thm:safety} and $\alpha+\tau\Delta+2\epsilon$ in Corollary~\ref{cor:mission}, with all other assumptions unchanged. This extension covers empty-set cells before discovery; undiscovered categories yielding nonempty sets remain outside the guarantee.

We use fallback $0$ in the main experiments to isolate the effect of ACI-calibrated semantic dilation.

\newpage
\section{Notation}
\label{app:notation}

\begin{table}[H]
\caption{Variable definitions.}
\label{tab:notation}
\centering
\small
\begin{tabular}{@{}l l@{}}
\toprule
\textbf{Symbol} & \textbf{Description} \\
\midrule
\multicolumn{2}{@{}l}{\textit{Environment and Robot}} \\
$\Omega \subset \mathbb{R}^2$ & Workspace domain \\
$x_t \!\in\! \mathcal{X} \!\subset\! \mathbb{R}^{n_x}$ & Robot state at time step $t$ \\
$u_t \!\in\! \mathcal{U} \!\subset\! \mathbb{R}^{n_u}$ & Control input \\
$(I_t, Z_t)$ & RGB-D observation: $I_t \!\in\! \mathbb{R}^{H \times W \times 3}$, $Z_t \!\in\! \mathbb{R}^{H \times W}$ \\
$\mathcal{J}$ & Set of indices for the grid map cells \\
$c_j \in \Omega$ & Center position of grid cell $j \in \mathcal{J}$ \\
$\delta$ & Single-step displacement bound \\
\midrule
\multicolumn{2}{@{}l}{\textit{Open-Vocabulary Semantics and Safety}} \\
$\mathcal{V}_t$ & Open-vocabulary semantic label set at time $t$ \\
$D_{\text{safe}}\!:\!\mathcal{V}_t \!\to\! \mathbb{R}_{\geq0}$ & Safety distance function \\
$\bar{D}$ & Upper bound: $\bar{D} \ge \sup_{l \in \mathcal{V}} D_{\text{safe}}(l)$ \\
$\phi\!:\!\mathbb{R}^{H \times W \times 3} \!\to\! \mathbb{R}^{H' \times W' \times d_{\text{em}}}$ & Dense VLM image encoder \\
$\psi\!:\!\mathcal{V} \!\to\! \mathbb{R}^{d_{\text{em}}}$ & VLM text encoder \\
$\mathbf{v}_j^{(t)} \in \mathbb{R}^{d_{\text{em}}}$ & Visual embedding for cell $j$ at time $t$ \\
$\mathbf{v}_{\text{cal},j} = \psi(l_j)$ & Text embedding of label $l_j$ \\
$f_{\text{expert}}$ & Delayed expert oracle; delay $\tau$ \\
\midrule
\multicolumn{2}{@{}l}{\textit{Neighborhood and Nonconformity Scores}} \\
$r$ & Neighborhood radius; $r \ge \bar{D} + \delta$ \\
$\mathcal{N}_r(x)$ & $\{j \in \mathcal{J} : \|x - c_j\| \le r\}$ \\
$s(\mathbf{v}, \mathbf{v}_{\text{cal}})$ & Pairwise nonconformity score; $s\!:\!\mathbb{R}^{d_{\text{em}}} \!\times\! \mathbb{R}^{d_{\text{em}}} \!\to\! \mathbb{R}_{\ge 0}$ \\
$S_t$ & Neighborhood NCS: $\max_{j \in \mathcal{N}_r(x_t)} s(\mathbf{v}_j^{(t)}, \mathbf{v}_{\text{cal},j}^{(t)})$ \\
\midrule
\multicolumn{2}{@{}l}{\textit{Adaptive Conformal Prediction}} \\
$Q_t$ & Adaptive nonconformity threshold at time $t$ \\
$\mathcal{C}_j^{(t)} \!\subseteq\! \mathcal{V}_t$ & Coverage set: $\{ l \!\in\! \mathcal{V}_t \mid s(\mathbf{v}_j^{(t)}, \psi(l)) \!\le\! Q_t \}$ \\
$\alpha,\;\alpha_t$ & Target / adaptive miscoverage rate \\
$\gamma$ & ACI step size \\
\midrule
\multicolumn{2}{@{}l}{\textit{Safety-Critical Planning}} \\
$d_{\text{wc}}^{(j)}$ & Worst-case distance: $\max_{l \in \mathcal{C}_j^{(t)}} D_{\text{safe}}(l)$ \\
$\mathcal{X}_{\text{safe}}^{(t)}$ & Safe configuration space \\
\bottomrule
\end{tabular}
\end{table}

\end{document}